\documentclass[11pt]{article}
\usepackage{amsmath,amsbsy,amsfonts,amssymb,amsthm,fullpage,units}
\usepackage{graphicx}
\usepackage{tablefootnote}
\usepackage{color,cases}
\usepackage{tikz}
\usetikzlibrary{calc,shapes}
\usepackage{mathtools}

\usepackage{subfigure}
\usepackage[margin=1in]{geometry}
\usepackage{xcolor}
\definecolor{DarkGreen}{rgb}{0.1,0.5,0.1}
\definecolor{DarkRed}{rgb}{0.5,0.1,0.1}
\definecolor{DarkBlue}{rgb}{0.1,0.1,0.5}
\usepackage{hyperref}
\usepackage{nicefrac}
\usepackage{algorithm}

\usepackage{algorithmic}
\usepackage{xspace}
\usepackage{wrapfig}
\makeatletter
\newtheorem*{rep@theorem}{\rep@title}
\newcommand{\newreptheorem}[2]{	\newenvironment{rep#1}[1]{		\def\rep@title{#2 \ref{##1}}		\begin{rep@theorem}}		{\end{rep@theorem}}}
\makeatother
\newtheorem{theorem}{Theorem}[section]
\newreptheorem{theorem}{Theorem}

\newtheorem{lemma}[theorem]{Lemma}
\newtheorem{claim}[theorem]{Claim}
\newtheorem{corollary}[theorem]{Corollary}

\newtheorem{proposition}[theorem]{Proposition}
\theoremstyle{definition}

\newtheorem{definition}[theorem]{Definition}

\numberwithin{equation}{section}

\newcommand{\IGNORE}[1]{}

\newcommand{\citep}{\cite}
\newcommand{\citet}{\cite}

\makeatother

\newcommand\E{\mathbb{E}}

\newcommand\wh{\widehat}

\newcommand\inner[1]{\langle #1 \rangle}

\newcommand\poly{\operatorname{poly}}

\newcommand\Proj{\operatorname{Proj}}

\newcommand{\Exp}{\mathop{\mathbb E}\displaylimits}

\newcommand{\allzero}{{\bf 0}}
\newcommand{\allones}{{\bf 1}}
\newcommand{\pmi}{\textup{PMI}}
\newcommand{\pmitensor}{\textup{PMIT}}

\newcommand{\norm}[1]{\lVert#1\rVert}
\newcommand{\Norm}[1]{\left\lVert#1\right\rVert}
\newcommand{\Id}{\textup{Id}}

\newcommand{\Qminushalf}{(Q^+)^{1/2}}
\newcommand{\minushalf}[1]{(#1^+)^{1/2}}
\newcommand{\spb}{\textup{spectrally bounded~}}

\newcommand{\ac}{\textup{approximately orthonormal~}}

\newcommand{\mper}{\,.}

\newcommand\R{\mathbb{R}}
\newcommand\N{\mathcal{N}}

\newcommand{\weight}{W}
\newcommand{\spr}{p}
\newcommand{\bs} {s}
\newcommand{\bd} {d}
\newcommand{\bw} {W}
\newcommand{\bfr} {F}
\newcommand{\svdU}{U} 
\newcommand{\svdV}{V} 
\newcommand{\subsvd}{S}

\newcommand{\prior}{\rho}
\newcommand{\fo}{A}
\newcommand{\so}{\Gamma}

\newcommand{\taylor}{P}
\newcommand{\samples}{N}

\let\epsilon=\varepsilon

\numberwithin{equation}{section}

\newcommand\MYcurrentlabel{xxx}

\newcommand{\MYstore}[2]{	\global\expandafter \def \csname MYMEMORY #1 \endcsname{#2}}

\newcommand{\MYload}[1]{	\csname MYMEMORY #1 \endcsname}

\newcommand{\MYnewlabel}[1]{	\renewcommand\MYcurrentlabel{#1}	\MYoldlabel{#1}}

\newcommand{\MYdummylabel}[1]{}

\newcommand{\torestate}[1]{		\let\MYoldlabel\label	\let\label\MYnewlabel	#1	\MYstore{\MYcurrentlabel}{#1}		\let\label\MYoldlabel}

\newcommand{\restatetheorem}[1]{		\let\MYoldlabel\label
	\let\label\MYdummylabel
	\begin{theorem*}[Restatement of \prettyref{#1}]
		\MYload{#1}
	\end{theorem*}
	\let\label\MYoldlabel
}

\newcommand{\restatelemma}[1]{		\let\MYoldlabel\label
	\let\label\MYdummylabel
	\begin{lemma*}[Restatement of \prettyref{#1}]
		\MYload{#1}
	\end{lemma*}
	\let\label\MYoldlabel
}

\newcommand{\restateprop}[1]{		\let\MYoldlabel\label
	\let\label\MYdummylabel
	\begin{proposition*}[Restatement of \prettyref{#1}]
		\MYload{#1}
	\end{proposition*}
	\let\label\MYoldlabel
}

\newcommand{\restatefact}[1]{		\let\MYoldlabel\label
	\let\label\MYdummylabel
	\begin{fact*}[Restatement of \prettyref{#1}]
		\MYload{#1}
	\end{fact*}
	\let\label\MYoldlabel
}

\newcommand{\restate}[1]{		\let\MYoldlabel\label
	\let\label\MYdummylabel
	\MYload{#1}
	\let\label\MYoldlabel
}

\def\shownotes{0}  \ifnum\shownotes=1
\newcommand{\authnote}[2]{$\ll$\textsf{\footnotesize #1 notes: #2}$\gg$}
\else
\newcommand{\authnote}[2]{}
\fi

\title{Provable learning of Noisy-or Networks}
\author{Sanjeev Arora\thanks{Princeton University, Computer Science Department. \texttt{arora@cs.princeton.edu}} \and Rong Ge\thanks{Duke University, Computer Science Department. \texttt{rongge@cs.duke.edu}} \and Tengyu Ma \thanks{Princeton University, Computer Science Department. \texttt{tengyu@cs.princeton.edu}}\and Andrej Risteski\thanks{Princeton University, Computer Science Department. \texttt{risteski@princeton.edu}}}
\begin{document}
	\maketitle
\newcommand{\bern}{\textup{Ber}}
\begin{abstract}
	\normalsize
Many machine learning applications use latent variable models to explain structure in data, whereby  visible variables (= coordinates of the given datapoint) are explained as a probabilistic function of some hidden variables. Finding parameters with the maximum likelihood is NP-hard even in very simple settings. In recent years, provably efficient algorithms were nevertheless developed for models with linear structures: topic models, mixture models, hidden markov models, etc. These algorithms use matrix or tensor decomposition, and make some reasonable assumptions about the parameters of the underlying model. 

But matrix or tensor decomposition seems of  little use when the latent variable model has nonlinearities. The current paper shows how to make progress: tensor decomposition is applied for learning the single-layer {\em noisy or} network, which is a textbook example of a Bayes net, and used for example in the classic QMR-DT  software for diagnosing which disease(s) a patient may have by observing the symptoms he/she exhibits.

The technical novelty here, which should be useful in other settings in future, is analysis of tensor decomposition in presence of systematic error (i.e., where the noise/error is correlated with the signal, and doesn't decrease as number of samples goes to infinity). This requires rethinking all steps of tensor decomposition methods from the ground up. 

For simplicity our analysis is stated assuming that the network parameters were chosen from a probability distribution but the method seems more generally applicable. 
 
\end{abstract}

\pagenumbering{gobble}

\newpage

\pagenumbering{arabic}

\section{Introduction}
\label{s:intro}

Unsupervised learning is important and potentially very powerful because of the availability of  the huge amount of unlabeled data --- often several orders of magnitudes more than the labeled data in many domains. {\em Latent variable models},  a popular approach in unsupervised learning,  model the latent structures in data: the \textquotedblleft structure\textquotedblright corresponds to some hidden variables, which probabilistically determine the values of the visible coordinates in data. 
{\em Bayes nets}  model the dependency structure of latent and observable variables via a directed graph. 
Learning parameters of a latent variable model given data samples is often seen as a {\em canonical} definition of unsupervised learning. Unfortunately, finding parameters with the maximum likelihood is NP-hard even in very simple settings. 
However, in practice many of these models can be learnt reasonably well using algorithms without polynomial runtime guarantees, such as expectation-maximization algorithm, Markov chain Monte Carlo, and variational inference.  
Bridging this gap between theory and practice is an important research goal.

\sloppy Recently it has become possible to use matrix and tensor decomposition methods to design polynomial-time algorithms to learn some simple latent variable models such as topic models~\cite{arora2012learning,arora2013practical}, sparse coding models~\cite{arora2015simple,MSS16}, mixtures of Gaussians~\cite{hsu2013learning,ge2015learning}, hidden Markov models~\cite{mossel2005learning}, etc. These algorithms are guaranteed to work if the model parameters satisfy some  conditions, which are reasonably realistic. In fact, matrix and tensor decomposition are a natural tool to turn to since they appear to be a sweet spot for theory whereby non-convex NP-hard problems can be solved provably under relatively clean and interpretable assumptions. But the above-mentioned recent results suggest that such methods apply only to solving latent variable models that are linear: specifically, they need the marginal of the observed variables conditioned on the hidden variables to depend linearly on the hidden variables. But many settings seem to call for nonlinearity in the model.  For example,  Bayes nets in many domains involve highly nonlinear operations on the latent variables, and could even have multiple layers. The study of neural networks also runs into nonlinear models such as restricted Boltzmann machines (RBM)~\cite{smolensky1986information,hinton2006reducing}. Can matrix factorization (or related tensor factorization) ideas help for learning nonlinear models?

This paper takes a first step by developing methods to apply tensor factorization to learn possibly the simplest nonlinear model, a {\em single-layer noisy-or network.}
This is a direct graphical model with hidden variable $d\in \{0,1\}^m$, and observation node $s\in \{0,1\}^n$. The hidden variables $d_1,\dots, d_m$ are independent and assumed to have Bernoulli distributions. The conditional distribution $\Pr[s \vert d]$ is parameterized by a non-negative weight matrix $W\in \R^{n\times m}$. We use $W^{i}$ to denote the $i$-row of $W$. Conditioned on $d$, the observations $s_1,\dots, s_n$ are assume to be independent  with distribution
\begin{align}
\label{eqn:noisyor1}\Pr\left[s_i = 0\mid d\right] = \prod_{j=1}^m \exp(-W_{ij}d_j) =\exp(-\inner{W^i, d})\mper
\end{align}
We see that $1-\exp(-W_{ji}d_j)$ can be thought of as the probability that $d_j$ activates symptom $s_i$, and $s_i$ is activated if one of $d_j$'s activates it --- which explains the name of the model, noisy-or.  
It follows that the conditional distribution $s\mid d$ is 
\begin{equation*}
\Pr[\bs \mid \bd] = \prod_{i=1}^n \left(1-\exp(-\langle \bw^i, \bd \rangle)\right)^{s_i} \left(\exp(-\langle \bw^i, \bd \rangle)\right)^{1-s_i} \mper
\end{equation*}

One canonical use of this model is to model the relationship between diseases and symptoms, as in the classical human-constructed tool for medical diagnosis called {\em Quick Medical  Reference} (QMR-DT) by (Miller et al.\cite{miller1982internist}, Shwe et al. \cite{shwe1991empirical}) This textbook example (\cite{jordan1999introduction}) of a Bayes net captures relationships between $570$ binary disease variables (latent variables) and $4075$ observed binary symptom variables, with $45,470$ directed edges, and the $W_{ij}$'s are small integers.\footnote{We thank Randolph Miller and Vanderbilt University for providing the current version of this network for our research.} The name \textquotedblleft noisy-or \textquotedblright derives from the fact that the probability that the OR of $m$ independent binary variables 
$y_1, y_2, \ldots, y_m$  is $1$  is exactly $1 -\prod_j(\Pr[y_j =0])$.  
  Noisy-or models are implicitly using this expression; specifically, for the $i$-th symptom we are considering the OR of $m$  events where the $j$th event is \textquotedblleft Disease $j$ does not cause symptom $i$\textquotedblright\, and its probability  is  $\exp(-\bw_{ij} d_j)$. Treating these events as independent leads to expression~(\ref{eqn:noisyor1}).
 
 The parameters of the QMR-DT network were hand-estimated by consulting human experts, but it is an interesting research problem whether such networks can be created in an automated way using only samples of patient data (i.e., the $s$ vectors).  Previously there were no approaches for this that work even heuristically at the required problem size ($n = 4000$). (This learning problem should not be confused with the simpler problem of {\em infering} the latent variables given the visible ones, which is also hard but has seen more work, including reasonable heuristic methods~\cite{jordan1999introduction}).
Halpern et al.\cite{halpern2013unsupervised,jernite2013discovering} have designed some algorithms for this problem. However, their first paper \cite{halpern2013unsupervised} assumes the graph structure is given. The second paper \cite{jernite2013discovering} requires the Bayes network to be quartet-learnable, which is a strong assumption on the structure of the network.
Finally, the problem of finding a ``best-fit'' Bayesian network according to popular metrics\footnote{Researchers resort to these metrics as when the graph structure is unknown, multiple structures may have the same likelihood, so maximum likelihood is not appropriate.} has been shown to be NP-complete by \cite{chickering1996learning} even when all of the hidden variables are also observed.  

\paragraph{Our algorithm and analysis.} Our algorithm uses Taylor expansion ---on a certain correlation measure called PMI, whose use in this context is new---to convert the problematic exponential into an infinite sum, where we can ignore all but the first two terms.
This brings the problem into the realm of tensor decomposition but with several novel twists having to do with systematic error (see the overview in Section~\ref{s:overview}). Our recovery algorithm 
makes several assumptions about $W$, the matrix of connection weights, listed in Section~\ref{sec:main_results}. We verified some of these on the QMR-DT network, but the other assumptions are asymptotic in nature. Thus the cleanest description of our algorithm is in a clean average-case setting.
First, we assume all latent variables are iid Bernoulli $\bern(\rho)$ for some $\rho$, which should be thought of as small  (In the  QMR-DT application, $\rho$ is like $O(1/m)$.)
Next we assume that the ground truth $W\in \R^{n\times m}$ is created by nature by picking its entries in iid fashion using the following random process:
\[
\bw_{ij} = \begin{cases}
0, & \text{with probability } 1 - p \\
\widetilde{W}_{ij}, & \text{with probability } p 
\end{cases}
\] 
where $\widetilde{W}_{ij}$'s are upper bounded by  $\nu_u$
for some constant $\nu_u$ and are identically distributed according to a distribution $\mathcal{D}$ which satisfies that for some constant $\nu_l > 0$, 
\begin{align}\Exp_{\widetilde{W}_{ij}\sim \mathcal{D}}\left[\exp(-\widetilde{W}^2_{ij})\right] \le 1 -\nu_l \mper\label{eqn:tildeWij}
\end{align}
The condition~\eqref{eqn:tildeWij} intuitively requires that $\widetilde{W}_{ij}$ is bounded away from 0. We will assume that $p\le 1/3$ and $\nu_u =O(1), \nu_l = \Omega(1)$. (Again, these are realistic for QMR-DT setting).

\begin{theorem}[Informally stated]
	There exists a polynomial time algorithm (Algorithm~\ref{alg:main}) that, given polynomially many samples from the noisy OR network described in the previous paragraph, recovers the weight matrix $W$ with $\widetilde{O}(\rho\sqrt{pm})$ relative error in $\ell_2$-norm in each column. 
\end{theorem}

Recall that we mostly thought of the prior of the diseases $\rho$ as being on the order $O(1/m)$. This means that even if $p$ is on the order of $1$, our relative error bound equals to $O(1/\sqrt{m}) \ll 1$.

\section{Preliminaries and overview}\label{s:overview}
We denote by $\allzero$ the all-zeroes vector and  $\allones$ the all-ones vector. 
$A^{+}$ will denote the Moore-Penrose pseudo-inverse of a matrix $A$, and for symmetric matrices $A$, we use $A^{-1/2}$ as a shorthand for $(A^+)^{1/2}$. The least \textit{non-zero} singular value of matrix $A$ is denoted $\sigma_{\min}(A)$.
 
For matrices $A, B$ we define the Kronecker product $\otimes$ as $(A\otimes B)_{ijkl} = A_{ij}B_{kl}.$ A useful identity is that $(A \otimes  B)\cdot (C \otimes D) = (AC) \otimes (BD)$ whenever the matrix multiplications are defined. 
Moreover, $A_i$ will denote the i-th column of matrix $A$ and $A^i$ the i-th row of matrix $A$.
 
 We write $A\lesssim B$ if there exists a universal constant $c$ such that $A\le cB$ and we define $\gtrsim $ similarly.  

 The pointwise mutual information 
of two binary-valued random variables $x$ and $y$ is $PMI2(x, y) \triangleq \log \frac{\Exp[xy]}{\Exp[x]\Exp[y]}$. Note that it is positive
iff $\Exp[x, y] > \Exp[x]\Exp[y]$ and thus is used as a measure of correlation in many fields. This concept can be extended in more than one way to a triple of boolean variables $x, y, z$ and  we use 
\begin{equation}
PMI3(x, y, z) \triangleq \log \frac{\Exp[xy]\Exp[yz]\Exp[zx]}{\Exp[xyz]\Exp[x]\Exp[y]\Exp[z]}. \label{eqn:PMItriple}
\end{equation}
(We will sometimes shorten PMI3 and PMI2 to PMI when this causes no confusion.)
\subsection{The Algorithm in a nutshell}
\label{subsec:taylorexpansion}
Our algorithm is given polynomially many samples from the model (each sample describing which  symptoms are or are not present in a particular patient). It starts by computing the following matrix $n\times n$ $\pmi$ and  and 
$n\times n \times n$ tensor $\pmitensor$, which tabulate
the correlations among all pairs and triples of symptoms (specifically, the indicator random variable for the symptom being absent):
\begin{align}
\pmi_{ij} &\triangleq  PMI2(1-s_i, 1-s_j).   \\
\pmitensor_{i, j, k} & \triangleq  PMI3(1-s_i, 1-s_j, 1-s_j)
\end{align}

The next proposition makes the key observation that the above matrix and tensor are close to rank $m$, which we recall is much smaller than $n$. Here
and elsewhere, for a matrix or vector $A$ we use $\exp(A)$ for  the matrix or vector obtained by taking the \textit{entries-wise} exponential. For convenience, we define $F, G\in \R^{n\times m}$ as
\begin{align}
F & \triangleq 1- \exp(-W) \label{eqn:def-F}\\
G &\triangleq 1-\exp(-2W)\mper \label{eqn:def-G}
\end{align}
 
\begin{proposition}[Informally stated]\label{prop:lowrank} Let  $F_k, G_k$ denote the $k$th columns of the above $F, G$. 
	Then, \begin{align}
\pmi & \approx \rho \left(FF^{\top} + \rho GG^{\top}\right) \label{eqn:2}
\quad = \quad \rho \sum_{k=1}^m  F_kF_k^{\top} + \rho^2 \sum_{k=1}^m   G_k G_k^{\top}  \\
\\	\pmitensor &\approx  \rho \Big(\underbrace{\sum_{k=1}^m   \bfr_k \otimes \bfr_k \otimes \bfr_k}_{:=S} + \underbrace{\rho \sum_{k=1}^m   G_k\otimes G_k \otimes G_k}_{:=E}\Big)\mper \label{eqn:3}
	\end{align}
\end{proposition}
The  proposition is proved later (with precise statement) in  Section~\ref{sec:pmi} by computing the moments by marginalization and using Taylor expansion to approximate the log of the moments, and ignoring terms $\rho^3$ and smaller. (Recall that $\rho$ is the probability that a patient has a particular disease, which should be small, of the order of $O(1/n)$. The dependence of the final error upon $\rho$ appears in Section~\ref{sec:main_results}.) 
Since the tensor $\pmitensor$ can be estimated to arbitrary accuracy given enough samples, the natural idea to recover the model parameters $W$ is to use Tensor Decomposition. This is what our algorithm does as well, except the following difficulties have to be overcome.

\noindent{\em Difficulty 1:} Suppose in equation~\eqref{eqn:3} we view the first summand $S$, which is rank $m$ with components $F_k$'s
 as the {\em signal} term. In all previous polynomial-time algorithms for tensor decomposition, the tensor is  required to  have the form
 $\sum_{k=1}^m   \bfr_k \otimes \bfr_k \otimes \bfr_k + \text{\em noise}$.   To make our problem fit this template we could consider the second summand $E$ as the \textquotedblleft noise\textquotedblright, especially since it is multiplied by $\rho \ll 1$ which tends  to make $E$ have smaller norm than $S$.
 But this is naive and incorrect, since $E$ is a very structured matrix: it is more appropriate viewed as {\em systematic error}. (In particular this error doesn't go down in norm as the number of samples goes to infinity.)
 In order to do tensor decomposition in presence of such systematic error, we will  need both a delicate error analysis and a very robust tensor decomposition algorithm. These will be outlined in Section~\ref{subsec:bias}.

\noindent{\em Difficulty 2:} To get our problem into a form suitable for tensor decomposition requires a {\em whitening} step, which uses the robust estimate of the whitening matrix from the second moment matrix.
In this case, the whitening matrix has to be extracted out of the $\pmi$ matrix, which itself suffers from
a  systematic error. This also is not handled in previous works, and requires a delicate control of the error. See Section~\ref{subsec:whitening} for more discussion.

\noindent{\em Difficulty 3:} There is another source of inexactness in equation~\eqref{eqn:3}, namely the approximation is only true for those entries with distinct indices --- for example, the diagonal entry $\pmi_{ii}$ has completely different formula from that for $\pmi_{ij}$ when $i\neq j$. This  will complicate the algorithm, as described in Subsections~\ref{subsec:bias} and~\ref{subsec:whitening}.  

The next few Subsections sketch how to overcome these difficulties, and the details appear in the rest of the paper. 

\subsection{Recovering matrices in presence of systematic error}\label{subsec:relative}
In this Section we recall the classical method of approximately recovering a matrix given noisy estimates of its entries. We discuss how to adapt that method to our setting where the error in the estimates is {\em systematic} and does not go down even with many more samples. The next section sketches an extension of this method to tensor decomposition with systematic error. 
 
In the classical setting, there is an unknown $n\times n$ matrix $S$ of rank $m$ and we are given $S+E$ where $E$ is an error matrix.  The method to recover $S$ is to compute the best rank-$m$ approximation to $S+E$. 
The quality of this approximation was studied by Davis and Kahan~\cite{davis1970rotation} and Wedin~\cite{wedin1972perturbation}, and many subsequent authors. The
quality of the recovery depends upon the ratio $||E||/\sigma_m(S)$, where 
$\sigma_m(\cdot)$ denotes $m$-th largest singular value and $||\cdot ||$ denotes the spectral norm.  To make this familiar lemma fit our setting more exactly, we will phrase the problem as trying to recover a matrix $S$ 
given noisy estimate 
$SS^{\top} +E$. Now one can only recover $S$ up to rotation, and the following lemma describes the
error in the Davis-Kahan recovery. 
It also plays a key role in the error analysis of the usual algorithm for tensor decomposition. 

\begin{lemma} In the above setting, let $K, \widehat{K}$ the subspace of the top $m$ eigenvectors of $SS^{\top}$ and $SS^{\top} +E$. Let $\epsilon$ be such that $\|E\| \leq \epsilon \cdot \sigma_m(SS^{\top})$. Then
$	\Norm{\Id_{K} - \Id_{\widehat{K}}}\lesssim \epsilon $ where $\Id$ is the identity transformation on the subspace in question.
\end{lemma}

The Lemma thus treats $||E||/\sigma_m(SS^{\top})$ as the definition of {\em noise/signal} ratio.
Before we generalize the definition and the algorithm to handle systematic error it is good to get some intuition, from looking at (\ref{eqn:2}):  $\pmi \approx \rho(FF^T + \rho GG^T)$.
 Thinking of the first term as signal and the second as error, let's check how bad is the noise/signal ratio defined in Davis-Kahan. The \textquotedblleft signal\textquotedblright is $\sigma_m(FF^{\top})$, which is smaller than $n$ since the 
  trace of $FF^{\top}$ is of the order of $mn$ in our probabilistic model for the weight matrix. 
  The \textquotedblleft noise\textquotedblright\ is the norm of $\rho GG^{\top}$, which is large since the $G_k$'s are nonnegative vectors with entries of the order of 1, and therefore 
   the quadratic form $\inner{\frac{1}{\sqrt{n}}\allones, \rho GG^{\top}\frac{1}{\sqrt{n}}\allones}$ can be as large as $\rho \sum_k \inner{G_k, \frac{1}{\sqrt{n}}\allones}^2\approx \rho mn$. Thus the Davis-Kahan noise/signal ratio is $\rho m$, and so when $\rho m \ll 1$, it allows recovering the subspace of $F$ with error $O(\rho m)$. Note that this is a vacuous bound since $\rho$ needs to be at least $1/m$ so that the hidden variable $d$ contains 1 non-zero entry in average. 
We'll argue that this error is too pessimistic and  we can in fact drive the estimation error down to close to $\rho$.  
\begin{definition}[spectral boundedness]\label{def:spb}
	Let $n\ge m$.  Let $E \in \R^{n\times n}$ be a symmetric matrix and $S \in \R^{n\times m}$. Then, we say $E$ is $\tau$-\spb by $S$ if 
	\begin{align}
		E\preceq \tau (SS^{\top}+ \sigma_m(SS^{\top})\cdot\Id_{n}) \label{eqn:1}
	\end{align}
	The smallest such $\tau$ is the \textquotedblleft error/signal ratio\textquotedblright~for this 
	recovery problem. 
\end{definition}

 This definition differs from Davis-Kahan's because of the $\tau SS^{\top}$ term on the  right hand side of (\ref{eqn:1}).
  This allows, for any unit vector $x$, the quadratic form value 
 $x^T E x$ to be as large as $\tau (x^TSS^{\top} x  + \sigma_m(SS^{\top}))$.
Thus for example the $\allones$ vector no longer causes a large noise/signal ratio since both quadtratic forms $FF^{\top}$ and $GG^{\top}$ have large values on it. 

This new error/signal ratio is no larger than the Davis-Kahan ratio, but can potentially be much smaller.
Now we show how to do a better analysis of the Davis-Kahan recovery in terms of it.
The proof of this theorem appears in Section~\ref{sec:proofs_proof_overview}.

\begin{theorem}[matrix perturbation theorem for systematic error]\label{thm:relative-davis-kahan}
	Let $n\ge m$. Let $S\in \R^{n\times m}$ be  of full rank. Suppose positive semidefinite matrix $E\in \R^{n\times n}$ is $\epsilon$-\spb by $S\in \R^{n\times m}$ for $\epsilon \in (0,1)$.  Let $K, \widehat{K}$ the subspace of the top $m$ eigenvectors of $SS^{\top}$ and $SS^{\top} +E$. Then, 
	\begin{align}
		\Norm{\Id_{K} - \Id_{\widehat{K}}}\lesssim \epsilon \mper\nonumber
	\end{align}
\end{theorem}

Finally, we should consider what this new definition of noise/signal ratio achieves. The next proposition (whose proof appears in Section~\ref{s:random}) shows that that under the generative model for $W$ sketched earlier, $\tau = O(\log n)$.
Therefore, $\sqrt{\rho}G$ is $\tilde{O}(\rho)$-bounded by $F$, and the recovery error of the subspace of $F$ from $FF^{\top} +\rho GG^{\top}$ is $\tilde{O}(\rho)$ (instead of $O(\rho m)$ using Davis-Kahan). 
\begin{proposition}\label{prop:generative_noise_signal_ratio}
	Under the generative model for $W$, w.h.p, the matrix $G = 1-\exp(-2W)$ is $\tau$-\spb by $F=1-\exp(-W)$, with $\tau = \tilde{O}(1)$. 
\end{proposition}

 Empirically, we can compute the $\tau$ value for the weight matrix $W$ in the QMR-DT dataset~\cite{shwe1991empirical}, which is a textbook application of noisy OR network. For the QMR-DT dataset, $\tau$ is under $6$. This implies that the recovery error of the subspace of $F$ guaranteed by Theorem~\ref{thm:relative-davis-kahan} is bounded by $O(\tau \rho) \approx \rho$, whereas the error bound by Davis-Kahan is $O(\rho m)$.

\subsection{Tensor decomposition with systematic error}\label{subsec:bias}
\newcommand{\barQa}{\bar{Q}_a}
\newcommand{\barQb}{\bar{Q}_b}
\newcommand{\barQc}{\bar{Q}_c}

Now we extend the insight from the matrix case to tensor recovery under systematic error.
In turns out condition~\eqref{eqn:1} is also a good measure of error/signal for the tensor recovery problem of~\eqref{eqn:3}. Specifically, if $G$ is $\tau$-bounded by $F$, then we can recover the components $F_k$'s from the $\pmitensor$ with column-wise error $O(\rho \tau^{3/2}\sqrt{m})$. This requires a non-trivial algorithm (instead of SVD), and the additional gain is that we can recover $F_k$'s individually, instead of only obtaining the subspace with the PMI matrix.  

First we recall the prior state of the art for the error analysis of tensor decomposition with Davis-Kahan type bounds. 
The best error bounds involve measuring the magnitude of the noise matrix $Z$ in a new way.
For any $n_1 \times n_2 \times n_3$ tensor $T$, we define the $\norm{\cdot}_{\{1\}\{2, 3\}}$ norm as
\begin{align}
\Norm{T}_{\{1\}\{2, 3\}}:= \sup_{\substack{x\in \R^{n_1}, y\in \R^{n_2n_3}\\ \|x\|=1,\|y\|=1}} \sum_{\substack{i\in [n_1] \\(j,k)\in [n_2]\times [n_3]}}x_{i} y_{jk}T_{ijk}\mper \label{eqn:def-spectral}
\end{align}
Note that this norm is in fact the spectral norm of the flattening of the tensor (into a $n_1\times n_2n_3$ dimensional matrix). 
This norm is larger than the injective norm\footnote{The injective norm of the tensor $T$ is defined as $
	\Norm{T}_{\{1\}\{2, 3\}}:= \sup_{\substack{x\in \R^{n_1}, y\in \R^{n_2} z\in \R^{n_3}\\ \|x\|=1,\|y\|=1,\|z\|=1}} \sum_{\substack{i\in [n_1],  j\in [n_2], k\in [n_3]}}x_{i} y_{j}z_kT_{ijk}\mper \nonumber
$}, but recently~\cite{MSS16} shows that $\epsilon$-error in this norm implies $O(\epsilon)$-error in the recovery guarantees of the components,
whereas if one uses injective norm, the guarantees often pick up an dimension-dependent factor~\cite{anandkumar2014tensor}.  We define $\norm{\cdot}_{\{2\}\{1,3\}}$ norm similarly. As is customary in tensor decomposition, the theorem is stated for tensors of a special form, where the components $\{u_i\},\{v_i\},\{w_i\}$ are  orthonormal families of vectors. This can be ensured without loss of generality using a procedure called whitening that uses the  2nd moment matrix.

\begin{theorem}[{Extension of~\cite[Theorem 10.2]{MSS16}}]\label{thm:main_orthogonal_alg_overview}\label{thm:main_orthogonal_alg}
There is a polynomial-time algorithm (Algorithm~\ref{alg:orthogonal} later) which has the following guarantee. 
	Suppose tensor $T$ is of the form 
	\begin{align}
	T = \sum_{i=1}^{r}u_i\otimes v_i\otimes w_i + Z\nonumber
	\end{align}
where $\{u_i\},\{v_i\},\{w_i\}$ are three collections of orthonormal vectors in $\R^d$, and $\Norm{Z}_{\{1\}\{2, 3\}}\le \epsilon$, $\Norm{Z}_{\{2\}\{1, 3\}}\le \epsilon$. 
Then, it returns $\{(\tilde{u}_i,\tilde{v}_i,\tilde{w}_i)\}$ in polynomial time that is $O(\epsilon)$-close to $\{(u_i,v_i,w_i)\}$ in $\ell_2$ norm up to permutation. \footnote{Precisely, here we meant that there exists a permutation $\pi$ such that for every $i$, $\max\{\norm{\tilde{u}_{\pi(i)}-u_i}, \norm{\tilde{v}_{\pi(i)}-v_i}, \norm{\tilde{w}_{\pi(i)}-w_i}\}\le O(\epsilon)$} \end{theorem}

But in our setting the noise tensor has systematic error. An analog of  
Theorem~\ref{thm:relative-davis-kahan} in this setting is complicated because even the   whitening step 
is nontrivial.
Recall also the inexactness in Proposition~\ref{prop:lowrank} due to the diagonal terms, which we earlier called {\em Difficulty 3}. We address this difficulty in the algorithm by setting up the problem using a sub-tensor of the PMI tensor. Let $S_a,S_b,S_c$ be a uniformly random equipartition of the set of indices $[n]$. 
Let \begin{align}
a_k = F_{k,S_a}, \quad b_k = F_{k,S_b}, \quad c_k  = F_{k,S_c}\,,\label{eqn:def-abc}
\end{align} where $F_{k,S}$ denotes the restriction of vector $F_k$ to subset $S$. 
Moreover, let 
\begin{align}
\gamma_k = G_{k,S_a}, \quad \delta_k = G_{k,S_b}, \quad \theta_k  = F_{k,S_c}\mper\label{eqn:def-gamma}
\end{align}
Then, since the sub-tensor $\pmitensor_{S_a,S_b,S_c}$ only contains entries with distinct indices, we can use Taylor expansion (see Lemma~\ref{claim:moments}) to obtain that 
\begin{align}
\pmitensor_{S_a,S_b,S_c} = \rho \sum_{k\in [m]} a_k\otimes b_k\otimes c_k + \rho^2 \sum_{k\in [m]} \gamma_k\otimes \delta_k\otimes \theta_k + \textup{higher order terms}\mper \nonumber
\end{align}
Here the second summand on the RHS corresponds to the second order term in the Taylor expansion. It turns out that the higher order terms are multiplied by $\rho^3$ and thus have negligible Frobenius norm, and therefore discussion below  will focus on the first two summands. 

For simplicity, let $T =   \pmitensor_{S_a,S_b,S_c}$. Our goal is to recover the components $a_k,b_k,c_k$ from the approximate low-rank tensor $T$. 

The first step is to whiten the components $a_k$'s, $b_k$'s and $c_k$'s. Recall that $a_k = F_{k,S_a}$ is a non-negative vector. This implies the matrix $A = [a_1,\dots, a_m]$ 
must have a significant contribution in the direction of the vector $\allones$, and thus is far away from being well-conditioned. For the purpose of this section, we assume for simplicity that we can access the covariance matrix defined by the vector $a_k$'s, 
\begin{align}
\bar{Q}_a := AA^{\top} = \sum_{k\in [m]} a_k a_k^{\top}\mper  \label{eqn:Qdef}\end{align}
Similarly we assume the access of $\bar{Q}_b$ and $\bar{Q}_c$ which are defined analogously. In Section~\ref{subsec:whitening} we discuss how to obtain approximately these three matrices.

Then, we can  compute the whitened tensor by applying transformation $\minushalf{\barQa}, \minushalf{\barQb}, \minushalf{\barQc}$ along the three modes of the tensor $T$, 
\begin{align}
\minushalf{\barQa}\otimes \minushalf{\barQb}\otimes \minushalf{\barQc} \cdot T & = \rho \sum_{k\in [m]} \minushalf{\barQa}a_k\otimes \minushalf{\barQb}b_k\otimes \minushalf{\barQc} c_k\nonumber \\
& + \underbrace{\rho^2 \sum_{k\in [m]} \minushalf{\barQa}\gamma_k \otimes \minushalf{\barQb}\delta_k\otimes \minushalf{\barQc}\theta_k}_{:=Z} +\textup{negligible terms}\nonumber
\end{align}

Now the first summand is a low rank orthogonal tensor, since $\minushalf{\barQa}a_k$'s are orthonormal vectors. However, the term $Z$ is a systematic error and we use the following Lemma to control its $\norm{\cdot }_{\{1\}\{2,3\}}$ norm. 
\begin{lemma}\label{lem:norm-to-boundedness}
		Let $n\ge m$ and $A,B,C\in \R^{n\times m}$ be full rank matrices and let $\Gamma,\Delta, \Theta\in \R^{d\times \ell}$ .
 Let $\gamma_i,\delta_i,\theta_i$ be the $i$-th column of $\Gamma, \Delta, \Theta$, respectively. Let $\bar{Q}_a  = AA^{\top}, \bar{Q}_b = BB^{\top}, \barQc = CC^{\top}$.
 Suppose $\Gamma\Gamma^{\top}$ (and $\Delta\Delta^{\top}$, $\Theta\Theta^{\top}$) is $\tau$-spectrally bounded by $A$ (and $B$, $C$ respectively), then, 
	\begin{align}
		\Norm{\sum_{i\in [\ell]}  \minushalf{\barQa}\gamma_i\otimes\minushalf{\barQb} \delta_i\otimes  \minushalf{\barQc}\theta_i}_{\{1\}\{2, 3\}}\le (2\tau)^{3/2}\mper\nonumber
	\end{align}
\end{lemma}

Lemma~\ref{lem:norm-to-boundedness} shows that to give an upper bound on the $\Norm{\cdot }_{\{1\}\{2, 3\}}$ norm of the error tensor $Z$,
it suffices to  
show that the square of the components of the error, namely, $\Gamma\Gamma^{\top}, \Delta\Delta^{\top}, \Theta\Theta^{\top}$ are $\tau$-\spb by the components of the signal $A,B,C$ respectively. This will imply that $\norm{Z}_{\{1\}\{2, 3\}}\le (2\tau)^{3/2} \rho^2$. 

Recall that $A$ and $\Gamma$ are two sub-matrices of $F$ and $G$. We have shown that $GG^{\top}$ is $\tau$-spectrally bounded by $F$ in Proposition~\ref{prop:generative_noise_signal_ratio}. It follows straightforwardly that the random sub-matrices also have the same property. 
\begin{proposition}\label{prop:random-boundedness}
In the setting of this section, under the generative model for $W$, w.h.p, we have that $\Gamma\Gamma^{\top}$ is $\tau$-\spb by $A$ with $\tau = O(\log n)$. The same is true for the other two modes. 
\end{proposition}
\noindent Using Proposition~\ref{prop:random-boundedness} and Lemma~\ref{lem:norm-to-boundedness}, we have that 
\begin{align}
\Norm{Z}_{\{1\}\{2,3\}} \lesssim \rho^2 \log^{3/2}(n)\mper\nonumber\end{align}

\noindent Then using Theorem~\ref{thm:main_orthogonal_alg_overview} on the tensor $\minushalf{\barQa}\otimes \minushalf{\barQb}\otimes \minushalf{\barQc} \cdot T $, we can recover the components $\minushalf{\barQa}a_k$'s, $\minushalf{\barQb}b_k$'s, and $\minushalf{\barQc}c_k$'s. This will lead us to recover $a_k$,$b_k$ and $c_k$, and finally to recover the weight matrix $W$. 

\subsection{Robust whitening}\label{subsec:whitening}

In the previous subsection, we assumed the access to $\barQa,\barQb,\barQc$ (defined in (\ref{eqn:Qdef})) which turns out to be highly non-trivial. A priori, using equation~\eqref{eqn:2}, noting that $A = [F_{1,S_a},\dots, F_{m,S_a}]$, we have
\begin{align}
\pmi_{S_a,S_a}/\rho \approx \barQa + \textrm{error}\mper\nonumber
\end{align}
However, this approximation can be arbitrarily bad for the diagonal entries of $\pmi$ since equation~\eqref{eqn:2} only works for entries with distinct indices. (Recall that this is why we divided the indices set into $S_a,S_b,S_c$ and studied the asymmetric tensor in the previous subsection). Moreover, the diagonal of the matrix $\barQa$ contributes to its spectrum significantly and therefore we cannot get meaningful bounds (in spectral norm) by ignoring the diagonal entries. 

This issue turns out to arise in most of the previous tensor papers and the solution was to compute $AA^{\top}$ by using the asymmetric moments $AB^{\top}, BC^{\top},CA^{\top}$, 
\begin{align}
AA^{\top} = (AB^{\top}) (CB^{\top})^{+}(CA^{\top})\mper\nonumber
\end{align}
Typically $AB^{\top}, BC^{\top},CA^{\top}$ can be estimated with arbitrarily small error (as number of samples go to infinity) and therefore the equation above leads to accurate estimate to $AA^{\top}$. However, in our case the errors in the estimate $\pmi_{\subsvd_a,\subsvd_b}\approx AB^{\top}$, $\pmi_{\subsvd_b,\subsvd_c}\approx BC^{\top}$, $\pmi_{\subsvd_c,\subsvd_a}\approx CA^{\top}$ are systematic. Therefore, we need to use a more delicate analysis to control how the error accumulates in the estimate,  
\begin{align}
\barQa \approx \pmi_{\subsvd_a,\subsvd_b}\cdot \pmi^{-1}_{\subsvd_b,\subsvd_c} \cdot \pmi_{\subsvd_c,\subsvd_a}\mper \nonumber
\end{align}

Here again, to get an accurate bound, we need to understand how the error in $\pmi_{\subsvd_a,\subsvd_b} -AB^{\top}$ behaves relatively compared with $AB^{\top}$ in a direction-by-direction basis. We generalized Definition~\ref{def:spb} to capture the asymmetric spectral boundedness of the error by the signal. 
 
\begin{definition}[Asymmetric spectral boundedness]\label{def:asymmetric}Let $n\ge m$ and $B,C\in \R^{n\times m}$.  We say a matrix  $E\in \R^{n\times n}$ is $\epsilon$-spectrally bounded by $(B,C)$ if $E$ can be written as:
	\begin{align}
	E = B\Delta_1 C^\top + B\Delta_2^\top + \Delta_3 C^\top + \Delta_4\mper\label{eqn:102}
	\end{align}
	Here $\Delta_1\in \R^{m\times m}$, $\Delta_2,\Delta_3 \in \R^{n\times m}$ and $\Delta_4\in \R^{n\times n}$ are matrices whose spectral norms are bounded by: $\|\Delta_1\| \le \epsilon$, $\|\Delta_2\|\le \epsilon\sigma_{min}(C)$, $\|\Delta_3\| \le \epsilon \sigma_{min}(B)$ and $\|\Delta_4\| \le \epsilon \sigma_{min}(B)\sigma_{min}(C)$.
\end{definition} 

Let $K$ be the column subspace of $B$ and $H$ be the column subspace of $C$. Then we have $\Delta_1 = B^+E(C^{\top})^+$, $\Delta_2= B^+E\Id_{H^{\perp}} $, $\Delta_3 = \Id_{K^{\perp}} E(C^{\top})^+$, $\Delta_4 = \Id_{K^{\perp}} E\Id_{H^{\perp}}$. Intuitively, they measure the relative relationship between $E$ and $B,C$ in different subspaces. For example, $\Delta_1$ is the relative perturbation in the column subspace of $K$ and row subspace of $H$.  When $B=C$, this is equivalent to the definition in the symmetric setting (this will be clearer in the proof of Theorem~\ref{thm:relative-davis-kahan}).

\begin{theorem}[Robust whitening theorem]
	\label{thm:whiten} 
		Let $n\ge m$ and $A,B,C\in \R^{n\times m}$. Suppose $\Sigma_{ab}, \Sigma_{bc},\Sigma_{ca}\in \R^{n\times n}$ are of the form, 
	$$
	\Sigma_{ab} = AB^\top + E_{ab},~~\Sigma_{bc} = BC^\top + E_{bc},~~\textup{and}~~\Sigma_{ca} = CA^\top + E_{ca}.
	$$
	where $E_{ab}, E_{bc}, E_{ca}$ are $\epsilon$-spectrally bounded by $(A,B)$, $(B,C)$, $(C,A)$ respectively.
	Then,  the matrix matrix  
	$$Q_a= \Sigma_{ab}[\Sigma_{bc}^\top]_m^+\Sigma_{ca}$$
	is a good approximation of $AA^{\top}$ in the sense that $Q_a = \Sigma_{ab}[\Sigma_{bc}^\top]_m^+\Sigma_{ca}-AA^{\top}$ is $O(\epsilon)$-\spb by $A$. Here $[\Sigma]_m$ denotes the best rank-$m$ approximation of $\Sigma$. 
\end{theorem}

The theorem is non-trivial even if the we have an absolute error assumption, that is, even if $\norm{E_{bc}}\le \tau\sigma_{min}(B)\sigma_{min}(C)$, which is stronger condition than $E_{bc}$ is $\tau$-\spb by $(B,C)$. 
Suppose we establish bounds on $\norm{\Sigma_{ab} - AB^{\top}}$, $\|\Sigma_{bc}^{+\top} - (BC^\top)^{+} \|$ and  $\norm{\Sigma_{ab} - AB^{\top}}$ individually, and then putting them together in the obvious way to control the error $\Sigma_{ab}[\Sigma_{bc}^\top]_m^+\Sigma_{ca} -  AB^{\top}(BC^{\top})^+ CA^{\top}$. Then the error will be too large for us. This is because standard matrix perturbation theory gives that $\|\Sigma_{bc}^{-\top} - (BC^\top)^{-1} \|$ can be bounded by $O\left(\|E_{bc}\|\|(BC^\top)^{-1}\|^2\right)\lesssim$
 $\epsilon/[\sigma_{min}(B)\sigma_{min}(C)]$, which is tight.  Then we multiply the error with the norm of the rest of the two terms, the error will be roughly $\epsilon\cdot \frac{\sigma_{\max}(B)\sigma_{\max}(C)}{\sigma_{\min}(B)\sigma_{\min}(C)}$. That is, we will loss a condition number of $B,C$, which can be dimension dependent for our case. 

The fix to this problem is to avoid bounding each term in $\Sigma_{ab}[\Sigma_{bc}^\top]_m^+\Sigma_{ca}$ individually. To do this, we will take the cancellation of these terms into account.  
 Technically, we re-decompose the product $\Sigma_{ab}[\Sigma_{bc}^\top]_m^+\Sigma_{ca}$ into a new product of three matrices $(\Sigma_{ab}B^{+}) (B[\Sigma_{bc}^\top]_m^+C)(C^+\Sigma_{ca})$, and then bound the error in each of these terms instead. See Section~\ref{s:whitening} for details. 

As a corollary, we conclude that the whitened vectors $(Q_a^+)^{1/2}a_i$'s are indeed approximately orthonormal.
\begin{corollary}
	In the setting of Theorem~\ref{thm:whiten}, we have that $
	(Q_a^+)^{1/2}A$ contains approximately orthonormal vectors as columns, in the sense that 
	\begin{align}
	\norm{(Q_a^+)^{1/2}AA^{\top}(Q_a^+)^{1/2}-\Id}\lesssim \epsilon\mper \nonumber
	\end{align}
\end{corollary}

Therefore we have found an approximate whitening matrix for $A$ even though we do not have access to the diagonal entries.

\section{Main Algorithms and Results}\label{sec:main_results}

As sketched in Section~\ref{s:overview}, our main algorithm (Algorithm~\ref{alg:main}) uses tensor decomposition on the PMI tensor. In this section, we describe the different steps and how the fit together. Subsequently, all steps will be analyzed in separate sections. 

\begin{algorithm}\caption{Learning Noisy-Or Networks via Decomposing PMI Tensor}\label{alg:main} 
	{\bf Inputs: } $N$ samples generated from a noisy-or network, disease prior $\rho$
	
	{\bf Outputs: } Estimate of weight matrix $\widehat{W}$. 

	\begin{enumerate}
		\item Compute the empirical PMI matrix and tensor $\widehat{\pmi}$, $\widehat{\pmitensor}$ using equation~\eqref{eq:pluginest}. 
		\item Choose a random equipartition $S_a, S_b, S_c$ of $[n]$. 
		\item Obtain approximate whitening matrices for $\widehat{\pmitensor}$ via Algorithm \ref{a:alg1} for the partitioning $S_a, S_b, S_c$
		\item Run robust tensor-decomposition Algorithm \ref{a:alg2} to obtain vectors $\hat{a}_i, \hat{b}_i, \hat{c}_i, i \in [m]$  
		\item 		 Let $Y_{i}$ be the concatenation of the three vectors $\allones - (\frac{1-\rho}{\rho})^{1/3}\hat{a}_i, \allones -(\frac{1-\rho}{\rho})^{1/3}\hat{b}_i,  \allones - (\frac{1-\rho}{\rho})^{1/3}\hat{c}_i$. (Recall that $\hat{a}_i, \hat{b}_i, \hat{c}_i$ are of dimension $n/3$ each.)
	  \item {\bf Return} $\widehat{W}$, where
		\[
		\widehat{W}_{i,j} = \begin{cases} 
		-\log ((Y_i)_j), & \text{if } (Y_i)_j > \exp(-\nu_u) : \\ 
		\exp(-\nu_u), & \text{otherwise } 
		\end{cases} 
		\]
			\end{enumerate}
\end{algorithm}

\begin{theorem}[Main theorem, random weight matrix] \label{thm:main-random}
	Suppose the true $W$ is generated from the random model in Section~\ref{s:intro} with $\prior \spr m \le c$ for some sufficiently small constant $c$. Then given $N = \poly(n,1/p/,1/\rho)$ number of examples, Algorithm~\ref{alg:main} returns 
a weight matrix $\widehat{W}$ in polynomial time  that satisfies 
	$$\forall i\in [m], \|\widehat{W}_i - W_i\|_2 \le \widetilde{O}(\eta\sqrt{\spr n})\,,	\label{t:finalguarantee}$$
	where 
	$\eta  = \tilde{O}\left( \sqrt{m \spr} \rho \right) $. 
\end{theorem} 
\noindent Note that the column $\ell_2$ norm of $W_i$ is on the order of $\sqrt{pn}$, and thus $\eta$ can be thought of as the relative error in $\ell_2$ norm. Note also that 
$\Pr[s_i=0] = 1 - \Pr[s_i=1] \approx  1 - \spr m \prior$, so $\prior \spr m = o(1)$ is necessary purely for sample complexity reasons.   
Finally, we can also state a result with a slightly weaker guarantee, but with only deterministic assumptions on the weight matrix $W$. Recall that $F = 1-\exp(-W)$ and $G = 1-\exp(-2W)$. 
We will also define third and fourth-order terms $H = 1 - \exp(-3W)$, $L = 1 - \exp(-4W)$.  

We also define the incoherence of a matrix $F$. Roughly speaking, it says that the left singular vectors of $F$ don't correlate with any of the natural basis vector much more than the average. 

\begin{definition}[Incoherence:]
	Let $F\in \R^{n\times m}$ have singular value decomposition $F = U\Sigma V^{\top}$. We say $F$ is $\mu$-incoherent if $\max_{i}\Norm{U_i} \le  \sqrt{\mu m/n}$.  where $U_i$ is the $i$-th row of $U$.
	\label{d:incoherence}
\end{definition}
We assume the weight matrix $W$ satisfies the following deterministic assumptions, \\
1.  $GG^{\top}, HH^{\top}, LL^{\top}$ is $\tau$-spectrally bounded by $F$ for $\tau \ge 1$. \\
2. $F$ is $\mu$-incoherent with $\mu \leq \widetilde{O}(\sqrt{n/m})$.\\
3. If $\max_i \|F_i\|_0 \leq p n$, with high probability over the choice of a subset $S_a, |S_a| = n/3$, $\sigma_{\min}(F_{S_a})\gtrsim \sqrt{np}$ and $\prior \spr m \le c$ for some sufficiently small constant $c$.  
\begin{theorem}[Main theorem, deterministic weight matrix] \label{thm:deterministic}	Suppose the matrix $W$ satisfies the conditions 1-3 above. Given polynomial number of samples, 
				Algorithm~\ref{alg:main} returns $\widehat{W}$ in polynomial time, s.t.  			$$\forall i\in [m], ~~\|\widehat{W}_i - W_i\|_2 \le \widetilde{O}(\eta \sqrt{np})\mper$$ 
	for 
	$\eta =  \sqrt{mp} \rho \tau^{3/2} $
					\label{t:finaldeterministic}
\end{theorem} 

Since the $\ell_2$ norm of $W_i$ is on the order of $\sqrt{np}$, the relative error in $\ell_2$-norm is as most $\sqrt{m}\rho \tau^{3/2}$, which mirrors the randomized case above.

The proofs of Theorems uses the overall strategy of Section~\ref{s:overview}, and is deferred to Section~\ref{s:missing_final}. We give a high level outline that demonstrates how the proofs depends on the machinery built in the subsequent sections.

Both Theorem~\ref{thm:main-random} and Theorem~\ref{thm:deterministic} are similarly proved -- the only technical difference being how the third and higher order terms are bounded. (Because of generative model assumption, for Theorem~\ref{thm:main-random} we can get a more precise control on them.) Hence, we will not distinguish between them in the coming overview.

Overall, we will follow the approach outlined in Section~\ref{s:overview}. Let us step through Algorithm~\ref{alg:main} line by line: 
\begin{enumerate}
\item The overall goal will be to recover the leading terms of the PMI tensor. Of course, we get samples only, so can merely get an empirical version of it. In Section~\ref{sec:sample}, we show that the simple plug-in estimator does the job -- and does so with polynomially many samples. 
\item Recall \emph{Difficulty 3} from Section~\ref{s:overview} : the PMI tensor and matrix expression is only accurate on the off-diagonal entries. In order to address this, in Section \ref{subsec:bias} we passed to a sub-tensor of the original tensor by partitioning the symptoms into three disjoint sets, and considering the induced tensor by this partition. 
\item In order to apply the robust tensor decomposition algorithm from Section~\ref{sec:tensor}, we need to first calculate whitening matrices. This is necessarily complicated by the fact that the diagonals of the PMI matrix are not accurate, as discussed in Section~\ref{subsec:whitening}. Section~\ref{s:whitening} gives guarantees on the procedure for calculating the whitening matrices.
\item This is main component of the algorithm: the robust tensor decomposition machinery. In Section~\ref{sec:tensor}, the conditions and guarantees for the success of the algorithm are formalized. There, we deal with the difficulties layed out in Section~\ref{subsec:relative} : namely that we have a substantial systematic error that we need to handle. (Both due to higher-order terms, and due to the missing diagonal entries)
\item This step, along with Step 6, is a post-processing step -- which allows us to recover the weight matrix $W$ after we have recovered the leading terms of the PMI tensor. 
\end{enumerate} 

We also give a short quantitative sense of the guarantee of the algorithm. (The reader can find the full proof in Section~\ref{s:missing_final}.)

To get quantitative bounds, we will first need a handle on spectral properties of the random model: these are located in Section~\ref{s:random}. 
As we mentioned above, the main driver of the algorithm is step 4, which uses our robust tensor decomposition machinery in Section~\ref{sec:tensor}. To apply the machinery, we first need to show that the second (and higher) order terms of the PMI tensor are spectrally bounded. This is done by applying Proposition~\ref{p:thirdspectral}, which roughly shows the higher-order terms are $O(\rho \log n)$-spectrally bounded by $\rho F F^{\top}$. The whitening matrices are calculated using machinery in Section~\ref{s:whitening}. We can apply these tools since the random model gives rise to a $O(1)$-incoherent $F$ matrix as shown in Lemma \ref{l:rowbound}. 

To get a final sense of what the guarantee is, the $l_2$ error which step 4 gives, via Theorem~\ref{thm:main_alg} roughly behaves like $\sqrt{\sigma_{\max}} \tau^{3/2}$, where $\sigma_{\max}$ is the spectral norm of the whitening matrices and $\tau$ is the spectral boundedness parameter. But, by Lemma~\ref{l:approx_evals} $\sigma_{\max}$ is approximately the spectral norm of $\rho F F^{\top}$ -- which on the other hand by Lemma~\ref{l:boundevals}  is on the order of $m n p^2 \rho$. Plugging in these values, we get the theorem statement.

\section{Finding the Subspace under Heavy Perturbations}\label{sec:proofs_proof_overview}
	\newcommand{\hA}{\widehat{A}}

In this section, we show even if we perturb a matrix $SS^\top$ with an error whose spectral norm might be much larger than $\sigma_{min}(SS^\top)$, as long as $E$ is spectrally bounded the top singular subspace of $S$ is still preserved. We defer the proof of the asymmetric case (Theorem~\ref{thm:whiten}) to Section~\ref{s:whitening}. We note that such type of perturbation bounds, often called relatively perturbation bounds, have been studied in~\cite{ipsen1998relative, li1998relative,li1998relativeII,li1997relativeiii}. The results in these papers either require the that signal matrix is full rank, or the perturbation matrix has strong structure. We believe our results are new and the way that we phrase the bound makes the application to our problem convenient. We recall Theorem~\ref{thm:relative-davis-kahan}, which was originally stated in Section~\ref{s:overview}. 

\begin{reptheorem}{thm:relative-davis-kahan} [matrix perturbation theorem for systematic error]
	Let $n\ge m$. Let $S\in \R^{n\times m}$ be  of full rank. Suppose positive semidefinite matrix $E\in \R^{n\times n}$ is $\epsilon$-\spb by $S\in \R^{n\times m}$ for $\epsilon \in (0,1)$.  Let $K, \widehat{K}$ the subspace of the top $m$ eigenvectors of $SS^{\top}$ and $SS^{\top} +E$. Then, 
	\begin{align}
	\Norm{\Id_{K} - \Id_{\widehat{K}}}\lesssim \epsilon \mper\nonumber
	\end{align}
\end{reptheorem}
\begin{proof}

	We can assume $\epsilon \le 1/10$ since otherwise the statement is true (with a hidden constant 10).  Since $E$ is a positive semidefinite matrix, we write $E =  RR^{\top}$ where $R = E^{1/2}$. 
	Since $A$ has full column rank, we can write $R = AS + B$ where $S\in \R^{m\times n}$ and the columns of $B$ are in the subspace $K^{\perp}$. (Specifically, we can choose $S = A^+R$ and $B = R-AA^+R = \Id_{K^{\perp}}B$.) By the definition of spectral boundedness, we have 
\begin{align}
BB^{\top} &= \Id_{K^{\perp}} RR^{\top}\Id_{K^{\perp}} \preceq \Id_{K^{\perp}} \epsilon\left(AA^{\top} + \sigma_m(AA^{\top})\Id_n\right)\Id_{K^{\perp}}\mper\nonumber\\
& = \epsilon\sigma_m(AA^{\top})\Id_{K^{\perp}}\mper\nonumber
\end{align}
Therefore, we have that $\norm{B}^2\le \epsilon \sigma_{\min}(AA^{\top})$. 
Moreover, we also have
	\begin{align}
	\Id_K RR^{\top} \Id_K\preceq \epsilon AA^{\top} + \epsilon \sigma_{\min}\Id_K\,,\nonumber\end{align}
	It follows that 
	\begin{align}
	ASS^{\top}A^{\top} \le 2\epsilon AA^{\top}\mper\nonumber
	\end{align}
	which implies $$\norm{SS^{\top}}\le \epsilon.$$ Let $P = \left(\Id_m + SS^{\top}\right)^{1/2}$. Then we write $AA^{\top} +E $ as,	\begin{align}
		AA^{\top} + E = AA^{\top} + RR^{\top} & = AA^{\top} + (AS+B)(AS+B)^{\top} \nonumber\\
		&= A(\Id+SS^{\top})A^{\top} + ASB^{\top} + BS^{\top}A^{\top} + BB^{\top}  \nonumber\\
	& = (AP + BS^{\top}P^{-1}) (AP + BS^{\top}P^{-1})^{\top} + BB^{\top} - BS^{\top}P^{-2}SB^{\top}\label{eqn:23}
	\end{align}
	Let $\widehat{A} = (AP + BS^{\top}P^{-1})$. Let $K'$ be the column span of $\widehat{A}$. We first prove that $\widehat{K}$ is close to $K'$. Note that 
	\begin{align}
	\Norm{BB^{\top} - BS^{\top}P^{-2}SB^{\top}}& \lesssim \Norm{B}^2 + \Norm{B}^2  \Norm{S^{\top}P^{-2}S}\lesssim \norm{B}^2 \tag{since $P= \Id +SS^{\top}\succeq SS^{\top} $}\\
	& \lesssim \epsilon \sigma_{\min}(AA^{\top})\mper\nonumber
	\end{align}
	Moreover, we have $\sigma_{\min}(\widehat{A}\widehat{A}^{\top}) = \sigma_{\min}(\widehat{A})^2 = \left(\sigma_{\min}(AP)- \norm{BS^{\top}P^{-1}}\right)^2\ge (1-O(\epsilon))\sigma_{\min}(A)^2$. Therefore, using  Wedin's Theorem (Lemma~\ref{lem:perturb-subspace}) on equation~\eqref{eqn:23}, we have that 
	\begin{align}
	\norm{\Id_{\widehat{K}}-\Id_{K'}}\lesssim \epsilon\mper\label{eqn:24}
	\end{align}
	Next we show $K'$ and $K$ are also close. We have 
\begin{align}
\norm{\hA - AP}\le\norm{BS^{\top}P^{-1}} \le \epsilon \sqrt{\sigma_{\min}(A)^2}\tag{since $\norm{S}\lesssim \sqrt{\epsilon}, \norm{B}\lesssim \sqrt{\epsilon}$}
\end{align}
 Therefore, by Wedin's Theorem, $K'$, as the span of top $m$ left singular vectors of $\hA$, is close to the span of the top left singular vector of $AP$, namely, $K$
	\begin{align}
	\norm{\Id_K - \Id_{K'}}\lesssim \epsilon\mper\label{eqn:25}
	\end{align}
Therefore using equation~\eqref{eqn:24} and~\eqref{eqn:25} and triangle inequality, we complete the proof. 
\end{proof}

\section{Robust Tensor Decomposition with Systematic Error} \label{sec:tensor}

In this section we discuss how to robustly find the tensor decomposition even in presence of systematic error. We first illustrate the main techniques in an easier setting of orthogonal tensor decomposition (Section~\ref{sec:tensorwarmup}), then we describe how it can be generalized to the general setting that we require for our algorithm (Section~\ref{sec:generaltensor}).

\subsection{Warm-up: Approximate Orthogonal Tensor Decomposition}
\label{sec:tensorwarmup}
We start with decomposing an orthogonal tensor with systematic error. The algorithm we use here is a slightly more general version of an algorithm in \cite{MSS16}.

\begin{algorithm}\caption{Robust orthogonal tensor decomposition}\label{alg:orthogonal} 
	{\bf Inputs: } Tensor $T \in \R^{d\times d\times d}$,  number $\delta, \epsilon\in (0,1)$. 
	
	{\bf Outputs: } Set $S = \{(\tilde{a}_i,\tilde{b}_i,\tilde{c}_i)\}$	
	\begin{enumerate}
		\item $S =\emptyset$
				\item {\bf For $s=1$ to $O(d^{1+\delta}\log d)$}		\item \quad\quad  Draw $g\sim \N(0,\Id_n)$, and compute $		M = \left(\Id_n\otimes \Id_n\otimes g^{\top}\right)\cdot T\mper$\footnotemark{}
		\item \quad\quad  Compute the top left and right singular vectors $u,v\in \R^d$ of $M$. Let $z = (u^{\top}\otimes v^{\top}\otimes \Id_n)\cdot T$. 
		\item \quad\quad  If $(u^{\top}\otimes v^{\top}\otimes z^{\top})\cdot T\ge 1-\zeta$, where $\zeta = O(\epsilon)$, and $u$ is $1/2$-far away from any of $u_i$'s with $(u_i,v_i,w_i)\in S$, then add $(u,v,w)$ to $S$. 
		\item {\bf Return} S
	\end{enumerate}
\end{algorithm}
\footnotetext{Recall that product of two tensor $(A\otimes B\otimes C) \cdot (E\otimes D\otimes F) = AE\otimes BD\otimes CF$}

\begin{theorem}[Stronger version of Theorem~\ref{thm:main_orthogonal_alg_overview}]\label{thm:main_orthogonal_technical}
	Suppose $\{u_i\},\{v_i\},\{w_i\}$ are three collection $\epsilon$-approximate orthonormal vectors. Suppose tensor $T$ is of the form 
	\begin{align}
	T = \sum_{i=1}^{r}u_i\otimes v_i\otimes w_i + Z\nonumber
	\end{align}
	with $\Norm{Z}_{\{2\}\{1,3\}}\le \tau$ and $\Norm{Z}_{\{1\}\{2,3\}}\le \tau$. 
	Then, with probability at least $0.9$, Algorithm~\ref{alg:orthogonal} returns  $S = \{(\tilde{u}_i,\tilde{v}_i,\tilde{w}_i)\}$ which is guaranteed to be  $O((\tau+\epsilon)/\delta)$-close to $\{(u_i,v_i,w_i)\}$ in $\ell_2$-norm up to permutation.  \end{theorem}
\begin{proof}[Proof Sketch of Theorem~\ref{thm:main_orthogonal_technical}]
	The Theorem is a direct extension of ~\cite[Theorem 10.2]{MSS16} to asymmetric and approximate orthogonal case. We only provide a proof sketch here. We start by writing 
	\begin{align}
	M = \left(\Id\otimes \Id \otimes g^{\top}\right)\cdot T	 = \underbrace{\sum_{i=1}^m \inner{g,w_i}u_iv_i^{\top}}_{:=M_s} + \underbrace{(\Id_n\otimes \Id_n \otimes g^{\top})\cdot Z}_{:=M_g}\label{eqn:200}
	\end{align}
	Since $\Norm{Z}_{\{2\}\{1,3\}}\le \tau$ and $\Norm{Z}_{\{1\}\{2,3\}}\le \tau$, \cite[Theorem 6.5]{MSS16} implies that with probability at least $1-d^2$ over the choice of $g$, 
	\begin{align}
	\Norm{(\Id_n\otimes \Id_n \otimes g^{\top})\cdot Z}\le 2\sqrt{\log d}\cdot \tau \nonumber
	\end{align}
	Let $t= 2\sqrt{\log d}$. We have that with probability $1/(d^{1+\delta}\log^{O(1)}d)$, $\inner{g,w_1}\ge (1+\delta/3)t$ and $\inner{g,w_j}\le t$ for every $j\neq 1$.  We condition on these events. Let $\bar{u}_i$ be a set of orthonormal vectors such that $E_u = [u_1,\dots, u_m] - [\bar{u}_1,\dots, \bar{u}_m]$ satisfies $\norm{E_u}\le \epsilon$ (we can take $\bar{u}_i$'s to be the whitening of $u_i$'s). Similarly define $\bar{v}_i$'s.  Then we have that the term (defined in equation~\eqref{eqn:200})  can be written as $\sum_{i}\inner{g,w_1}\bar{u}_i\bar{v}_i + E'$ where $\norm{E}'\lesssim \epsilon $. Let $\bar{M}_S = \sum_{i}\inner{g,w_1}\bar{u}_i\bar{v}_i$. Then $\bar{M}_S$ has top singular value  $\inner{g,w_1}\ge (1+\delta/3)t$, and second singular value at most $t$. Moreover, the term $M_g+E'$ has spectral norm bounded by $O(\tau+\epsilon)$.  Thus by Wedin's Theorem (Lemma~\ref{lem:perturb-subspace}), the top left and right singular vectors $u,v$ of $M_S+M_g = \bar{M}_S + M_g+E'$ are  $O((\tau+\epsilon)/\delta)$-close to $\bar{u}_1$ and $\bar{v}_1$ respectively. They are also $O((\tau+\epsilon)/\delta)$-close to $u_1,v_1$ since $u_1$ is close to $\bar{u}_1$.  Moreover, we have $(u^{\top}\otimes v^{\top}\otimes \Id)\cdot T$ is $O(\tau/\delta)$-close to $w_1$. 
	
	Therefore, with probability $1/(d^{1+\delta}\log^{O(1)}d)$, each round of the for loop in Algorithm~\ref{alg:orthogonal} will find $u_1,v_1,w_1$.  Line 5 is used to verify if the resulting vectors are indeed good using the injective norm as a test. It can be shown that if the test is passed then $(u,v,z)$ is close to one of the component. Therefore, after $d^{1+\delta}\log^{O(1)}d$ iterations, with high probability, we can find all of the components. 

\end{proof}

\subsection{General tensor decomposition}
\label{sec:generaltensor}

In many previous works, general tensor decomposition is reduced to orthogonal tensor decomposition via a whitening procedure. However, here in our setting we cannot estimate the exact whitening matrix because of the systematic error. Therefore we need a more robust version of approximate whitening matrix, which we define below:

\begin{definition}
	Let $r\le d$. A collection of $r$ vectors $\{a_1,\dots, a_r\}$ is $\epsilon$-\ac if the matrix $A$ with $a_i$ as columns satisfies 
	\begin{align}
	\Norm{A^{\top}A-\Id}\le \epsilon
	\end{align}
\end{definition}
\begin{definition}
	Let $d\ge r$ and $A = [a_1,\dots, a_r]\in \R^{d\times r}$. A PSD matrix $Q\in \R^{d\times d}$ is an $\epsilon$-approximate whitening matrix for $A$ if 
	$\Qminushalf A$ is $\epsilon$-\ac. 
\end{definition}

\begin{algorithm}\caption{Tensor decomposition with systematic error}\label{a:alg2} 
	{\bf Inputs: } Tensor $T \in \R^{n_1\times n_2\times n_3}$ and $\epsilon$-approximate whitening matrices $Q_a, Q_b, Q_c\in \R^{d\times d}$. 
	
	{\bf Outputs: } $\{\hat{a}_i,\hat{b_i},\hat{c}_i\}_{i\in [r]}$
	
	\begin{enumerate}
		\item Compute $\tilde{T} =  \minushalf{Q_a}\otimes \minushalf{Q_b}\otimes \minushalf{Q_c} \cdot T$
		\item Run orthogonal tensor decomposition (Algorithm~\ref{alg:orthogonal}) with input $\tilde{T}$, and obtain $\{\breve{a}_i,\breve{b}_i,\breve{c}_i\}$
		\item {\bf Return: } $\{Q_a^{1/2}\breve{a}_i,Q_b^{1/2}\breve{b}_i,Q_c^{1/2}\breve{c}_i\}$
			\end{enumerate}
\end{algorithm}
With this in mind, we can state the guarantee on the tensor decomposition algorithm (Algorithm~\ref{a:alg2}). 
\begin{theorem}\label{thm:main_alg}
		Let $d\ge r$, and $A,B,C\in \R^{d\times r}$ be full rank matrices. Let $\Gamma,\Delta, \Theta\in \R^{d\times \ell}$ . Let $a_i,b_i,c_i,\gamma_i,\delta_i,\theta_i$ be the columns of $A,B,C,\Gamma, \Delta, \Theta$ respectively. Suppose tensor $T$ is of the form 
		\begin{align}
	T & = \sum_{i=1}^{r} a_i\otimes b_i\otimes c_i + \sum_{i=1}^{\ell} \gamma_i\otimes \delta_i\otimes \theta_i + E  
	\end{align}
	Suppose  matrices $Q_a\in \R^{d\times d}, Q_b\in \R^{d\times d}, Q_c\in \R^{d\times d}$ are $\epsilon$-approximate whitening matrices for $A,B,C$, and suppose $\Gamma ,\Delta, \Theta $ are $\tau$-\spb by $Q_a^{1/2},Q_b^{1/2},Q_c^{1/2}$, respectively. 
	Then, Algorithm~\ref{a:alg2} returns $\hat{a}_i,\hat{b}_i,\hat{c}_i$ that are $O(\eta)$-close to $a_i,b_i,c_i$ in $\tilde{O}(d^{4+\delta})$ time with 
	$$\eta \lesssim  \max(\Norm{Q_a}, \Norm{Q_b}, \Norm{Q_c})^{1/2}  \cdot \left(\tau^{3/2}  + \sigma^{-3/2}\norm{E}_{\{1,2\}\{3\}}+\epsilon\right) \cdot 1/\delta$$
	where $\sigma = \min(\sigma_{\min}(Q_a), \sigma_{\min}(Q_b), \sigma_{\min}(Q_c)) $. \end{theorem}

Note that in our model, the matrix $E$ has very small spectral norm as it is the third order term in $\rho$ (and $\rho = O(1/n)$). The spectral boundedness of $\Gamma ,\Delta, \Theta$ are discussed in Section~\ref{s:random}. Therefore we can expect the RHS to be small.

In order to prove this theorem, we show after we apply whitening operation using the approximate whitening matrices, the tensor is still close to an orthogonal tensor.  To do that, we need the following lemma which is a useful technical consequence of the condition~\eqref{eqn:1}. 

\begin{lemma}\label{lem:consequence_spb} Suppose $F$ is $\tau$-\spb by $g$. Then, 
	\begin{align}
	\|G^{\top}(FF^{\top})^+ G\|\le 2\tau\mper\label{eqn:12}
	\end{align}
\end{lemma}

\begin{proof}
	Let $K$ be the column span of $F$. Let $Q = FF^{\top}$.  
	Multiplying $\Qminushalf$ on both sides of equation~\eqref{eqn:1}, we obtain that 
	\begin{align}
	\Qminushalf GG^{\top} \Qminushalf& \preceq \tau (\Id_K + \sigma_m(Q) Q^{+}) \nonumber\\
	& \preceq \tau (\Id_K + \sigma_m(Q)\|Q^+\|\Id_K) \nonumber\\
	& \preceq 	2\tau \Id_K	\nonumber	\end{align}
	\sloppy	It follows that $	\|\minushalf{Q}G\|\le \sqrt{2\tau}$, which in turns implies that $\|G^{\top}(FF^{\top})^+ G\|=\norm{G^{\top}\minushalf{Q}\minushalf{Q}G}\le 2\tau$.\end{proof}

We also need to bound the $\{1,2\}\{3\}$ norm of the following systematic error tensor. This is important because we want to bound the spectral norm of the perturbation after the whitening operation.
\begin{lemma}[{Variant of~\cite[Theorem 6.1]{MSS16}}]\label{lem:tensor_norm_bound} Let $\Gamma,\Delta, \Theta\in \R^{d\times \ell}$. Let $\gamma_i,\delta_i,\theta_i$ be the $i$-th column of $\Gamma, \Delta, \Theta$, respectively. Then, 
	\begin{align}
	\Norm{\sum_{i\in [\ell]}\gamma_i\otimes \delta_i\otimes \theta_i }_{\{1,2\}\{3\}}\le \Norm{\Gamma}\cdot\Norm{\Theta}\cdot \Norm{\Delta}_{1\to 2}\le\Norm{\Gamma}\cdot\Norm{\Theta}\cdot \Norm{\Delta} 	\end{align}
\end{lemma}
\begin{proof}[Proof of Lemma~\ref{lem:tensor_norm_bound}]
		Using the definition of $\norm{\cdot}_{\{1,2\}\{3\}}$ we have that 
	\begin{align}
	\Norm{\sum_{i\in []} \gamma_i\otimes \delta_i \otimes \theta_i  }_{\{1,2\}\{3\}} & =  	\Norm{\sum_{i\in [\ell]} (\gamma_i\otimes \delta_i) \theta_i^{\top}  } \\\nonumber
	& \le \Norm{\sum_{i\in [\ell]} (\gamma_i\otimes \delta_i)  (\gamma_i\otimes \delta_i)}^{1/2}\Norm{\sum_{i\in [\ell]} \theta_i \theta_i^{\top}}^{1/2} \tag{by Cauchy-Schwarz inequality}\\
	& = \Norm{\sum_{i\in [\ell]} (\gamma_i\gamma_i^{\top})\otimes (\delta_i\delta_i^{\top})}^{1/2}\Norm{\Theta} \nonumber
	\end{align}
	Next observe that we have that for any $i$, $\delta_i\delta_i^{\top}\preceq (\max\Norm{\delta_i}^2)\Id $ and therefore, 
	\begin{align}
	(\gamma_i\gamma_i^{\top})\otimes (\delta_i\delta_i^{\top}) \preceq \gamma_i\gamma_i^{\top} \otimes (\max\Norm{\delta_i}^2)\Id \mper
	\end{align}
	It follows that 
	\begin{align}
	\Norm{\sum_{i\in [r]} \gamma_i\otimes \delta_i \otimes \theta_i  }_{\{1,2\}\{3\}}	 & \le \Norm{\sum_{i\in [r]} \gamma_i\gamma_i^{\top} \otimes (\max\Norm{\delta_i}^2)\Id }^{1/2}\Norm{\Theta} \nonumber\\
	& = \Norm{\Gamma}\cdot\Norm{\Theta}\cdot \Norm{\Delta}_{1\to 2}\nonumber\mper
	\end{align}
	
\end{proof}

With this in mind, we prove the main theorem: 

\begin{proof}[Proof of Theorem~\ref{thm:main_alg}]
		Let $\tilde{A} = \minushalf{Q_a} A$, $\tilde{B} = \minushalf{Q_b} B$, $\tilde{C} = \minushalf{Q_c} C$. Moreover, let $\tilde{\Gamma} = \minushalf{Q_a} \Gamma$ and define $\tilde{\Delta}$, $\tilde{\Theta}$ similarly. Let $\tilde{a}_1,\tilde{b}_i,\tilde{c}_i, \tilde{\gamma}_i,\tilde{\delta}_i,\tilde{\theta}_i$ be their columns. Then we have that $\tilde{T}$ as defined in Algorithm \ref{a:alg2} satisfies
\begin{align}
\tilde{T}  =  \sum_{i=1}^{r} \tilde{a}_i\otimes \tilde{b}_i\otimes \tilde{c}_i + \sum_{i=1}^{\ell} \tilde{\gamma}_i\otimes \tilde{\delta}_i\otimes \tilde{\theta}_i + \tilde{E}  
\end{align}
where $\tilde{E} = \minushalf{Q_a}\otimes \minushalf{Q_b}\otimes \minushalf{Q_c}  \cdot E$. We will show that $\tilde{T}$ meets the condition of Theorem~\ref{thm:main_orthogonal_alg}. 
	Since $Q_a$ is an $\epsilon$-approximate whitening matrix of $A$, by Definition, $\tilde{A} = \minushalf{Q_a}A$ is $\epsilon$-\ac. Similarly, $\tilde{B},\tilde{C}$ are $\epsilon$-\ac. 
	
 $\Gamma$ is $\tau$-\spb by $Q_a$, hence by Lemma~\ref{lem:consequence_spb}, we have that $\norm{\tilde{\Gamma}}\le \sqrt{2\tau}$. Similarly, $\norm{\Theta}, \norm{\Delta}\le \sqrt{2\tau}$. Applying Lemma~\ref{lem:tensor_norm_bound}, we have, 
	\begin{align}
\Norm{\sum_{i=1}^{\ell} \tilde{\gamma}_i\otimes \tilde{\delta}_i\otimes \tilde{\theta}_i}_{\{1,2\}\{3\}}\le (2\tau)^{3/2}
	\end{align}
Moreover, we have $\norm{\tilde{E}}_{\{1,2\}\{3\}}\le \norm{\minushalf{Q_a}}\cdot  \norm{\minushalf{Q_c}}\cdot  \norm{\minushalf{Q_c}} \norm{E}_{\{1,2\}\{3\}}\le \sigma^{-3/2} \norm{E}_{\{1,2\}\{3\}}$, where $\sigma = \min \{\sigma_{\min}(Q_a), \sigma_{\min}(Q_b), \sigma_{\min}(Q_c) \}$. Therefore, using Theorem~\ref{thm:main_orthogonal_alg} (with $a_i,b_i,c_i$ there replaced by $\tilde{a}_i,\tilde{b}_i,\tilde{c}_i$, and $Z$ there replaced by $\sum_{i=1}^{\ell} \tilde{\gamma}_i\otimes \tilde{\delta}_i\otimes \tilde{\theta}_i+ \tilde{E}$), we have that a set of vectors  $\{\breve{a}_i,\breve{b}_i,\breve{c}_i\}$ that are $\epsilon$-close to $\{\tilde{a}_i,\tilde{b}_i,\tilde{c}_i\}$ with $\epsilon = (2\tau)^{3/2} + \sigma^{-3/2}\norm{\tilde{E}}_{\{1,2\}\{3\}}$. Therefore, we obtain that $\norm{a_i-Q^{1/2}\breve{a}_i} \le \norm{Q_a}^{1/2} \epsilon$. Similarly we can control the error for $b_i$ and $c_i$ and complete the proof. \end{proof}

\section{Conclusions}
We have presented theoretical progress on the longstanding open problem of presenting a polynomial-time algorithm for learning noisy-or networks given sample outputs from the network. In particular it is enouraging that linear algebraic methods like tensor decomposition can play a role. Earlier there were no good approaches for this problem; even heuristics fail for realistic sizes like $n=1000$. 

Can sample complexity be reduced, say to subcubic? (Cubic implies more than one billion examples for networks with $1000$ outputs.) Possibly this requires exploiting some hierarchichal structure --e.g. groupings of diseases and symptoms--- in practical noisy-OR networks  but exploring such possibilities using the current version of QMR-DT is difficult because it has been scrubbed of labels for diseases and symptoms.) 

Various more practical versions of our algorithm are also easy to conceive and will be tested in the near future. This could be somewhat analogous to topic modeling, for which discovery of provable polynomial-time algorithms soon led to very efficient algorithms.

\subsection*{Acknowledgments: } We thank Randolph Miller and Vanderbilt University for providing the current version of this network for our research.

\bibliography{ref,ref_tensor_2}
\bibliographystyle{alpha}
\appendix
\section{Formal expression for the PMI tensor} \label{sec:pmi}

In this section we formally derive the expressions for the PMI tensors and matrices, which we only informally did in Section~\ref{s:overview}. 

As a notational convenience for $l \in \mathbb{N}$, we will denote by $\tilde{\taylor}_l$ the matrix which has as columns the vectors $1-\exp(-l \bw_k), k \in [m]$.
Furthermore, for a subset $S_a \subseteq [n]$, we will introduce the notation 
$$\taylor_{l, S_a} = \sum_{k \in [m]} \left((\tilde{\taylor}_l)_{k,S_a}\right) \left((\tilde{\taylor}_l)_{k,S_a}\right)^{\top} = \sum_{k \in [m]} \left(1-\exp(-l \bw_k)_{S_a}\right) \left(1-\exp(-l \bw_{k})_{S_a}\right)^{\top}$$  
These matrices will appear naturally in the expressions for the higher-order terms in the Taylor expansion for the PMI matrix and tensor. 

We first compute the formally the moments of the noisy-or model. 

\begin{lemma} \label{claim:moments}
		We have
	\begin{align*}
	\log \Pr[s_i = 0]  &= \sum_{k\in [m]} \log \left(1 - \rho(1-\exp(\bw_{ik}))\right)\\
	\forall i\neq j \log \Pr\left[s_i = 0 \wedge s_j= 0\right] &=\sum_{k\in [m]} \log \left(1 - \rho(1-\exp(\bw_{ik}+\bw_{jk}))\right) \\
	\forall \textup{ distinct } i,j,k\in [n], \log \Pr\left[s_i = 0 \wedge s_j= 0\wedge s_{\ell} = 0\right] &=\sum_{k\in [m]} \log \left(1 - \rho(1-\exp(\bw_{ik}+\bw_{jk}+\bw_{\ell k}))\right) 
	\end{align*}
\end{lemma}

\begin{proof}[Proof of Lemma~\ref{claim:moments}]
	We only give the proof for the second equation. The rest can be shown analogously. 
	\begin{align*}
	\log \Pr\left[s_i = 0 \wedge s_j = 0\right] & = \Exp\left[\Pr[s_i = 0 \vert d]\cdot \Pr[s_j = 0 \vert d]\right] = \Exp\left[\exp(-(\bw_{i}+\bw_{j})^{\top}d)\right]\\
	& = \prod_{k\in [m]} \Exp\left[\exp(-(\bw_{ik}+\bw_{jk})d_k))\right] \\
	& = \prod_{k\in [m]} \left(1 - \rho(1-\exp(-(\bw_{ik}+\bw_{jk})))\right)\mper 
	\end{align*}
\end{proof}

With this in mind, we give the expression for the PMI tensor along with all the higher-order terms.

\begin{proposition} 

For any equipartition $S_a, S_b, S_c$ of $[n]$, the restriction of the PMI tensor $\pmitensor_{S_a, S_b, S_c}$ satisfies, 
for any $L \geq 2$, 
\small
\begin{equation} \pmitensor_{S_a, S_b, S_c} = \frac{\prior}{1-\prior} \sum_{k \in [m]} F_{k,S_a} \otimes F_{k,S_b} \otimes F_{k,S_c} + 
\sum_{l = 2}^L (-1)^{l+1} \left(\frac{1}{l}\left(\frac{\prior}{1-\prior}\right)^l\right) \sum_{k \in [m]} (\tilde{\taylor}_l)_{k,S_a} \otimes (\tilde{\taylor}_l)_{k,S_b} \otimes (\tilde{\taylor}_l)_{S_c} + E_L \label{ex:exprpmi} \end{equation} 
\normalsize
where 
$$\|E_L\|_{\{1,2\},\{3\}} \leq \frac{(mn)^3}{L} \frac{\left(\frac{\prior}{1-\prior}\right)^L}{1 - \left(\frac{\prior}{1-\prior}\right)^L}$$  

\label{p:firstspectral}
\end{proposition} 
\begin{proof}
The proof will proceed by Taylor expanding the log terms. Towards that, using Lemma \ref{claim:moments}, we have :
\small
\begin{align*} & \pmitensor_{ijl} = \\
& \sum_{k\in [m]} \log \frac{\left(1 - \prior (1-\exp(-\bw_{ik} - \bw_{jk}))\right) \left(1 - \prior (1-\exp(-\bw_{ik} - \bw_{lk}))\right) \left(1 - \prior (1-\exp(-\bw_{jk} -\bw_{lk}))\right)}{\left(1 - \prior (1-\exp(-\bw_{ik}-\bw_{jk} - \bw_{lk}))\right)\left(1 - \prior(1-\exp(-\bw_{ik}))\right)\left(1 - \prior(1-\exp(-\bw_{jk}))\right) \left(1 - \prior(1-\exp(-\bw_{lk}))\right)}\end{align*}
\normalsize
By the Taylor expansion of $\log (1 - x)$, we get that 
\begin{equation*}
\begin{multlined} \pmi_{ijl} =  \\ 
 -\sum_{t=1}^{\infty} \frac{1}{t} \sum_{k \in [m]} \prior^t (\left( \left(1-\exp(-\bw_{ik}-\bw_{jk})\right)\right)^t + \left( \left(1-\exp(-\bw_{ik}-\bw_{lk})\right)\right)^t + \left( \left(1-\exp(-\bw_{jk}-\bw_{lk})\right)\right)^t - \\
 \left(1-\exp(-\bw_{ik})\right)^t - \left(1-\exp(-\bw_{jk})\right)^t - \left(1-\exp(-\bw_{lk})\right)^t -  
 \left(1-\exp(-\bw_{ik} -\bw_{jk} - \bw_{lk})\right)^t ) \end{multlined}\end{equation*}
Furthermore, note that

\begin{equation*}
\begin{multlined}
\left( \left(1-\exp(-\bw_{ik}-\bw_{jk})\right)\right)^t + \left( \left(1-\exp(-\bw_{ik}-\bw_{lk})\right)\right)^t + \left( \left(1-\exp(-\bw_{jk}-\bw_{lk})\right)\right)^t - \\
 \left(1-\exp(-\bw_{ik})\right)^t - \left(1-\exp(-\bw_{jk})\right)^t - \left(1-\exp(-\bw_{lk})\right)^t -  
 \left(1-\exp(-\bw_{ik} - \bw_{jk} - \bw_{lk})\right)^t = \\ 
\sum_{l=1}^t {t \choose l} (-1)^l \left(1-\exp(-l \bw_{ik})\right) \left(1-\exp(-l \bw_{jk})\right) \left(1-\exp(-l \bw_{lk})\right)
\end{multlined} 
\end{equation*} 
by simple regrouping of the terms. By exchanging $l$ and $t$, we get

\small
\begin{align} \pmitensor_{S_a, S_b, S_c} & = \sum_{l=1}^{\infty} \sum_{t \geq l} (-1)^{l+1} \left( \prior^t \frac{1}{t} {t \choose l} \right) \sum_{k \in [m]} \left(1-\exp(-l \bw_k)\right)_{S_a} \otimes \left(1-\exp(-l \bw_k)\right)_{S_b} \otimes \left(1-\exp(-l \bw_k)\right)_{S_c}  \nonumber \\ 
& = \sum_{l=1}^{\infty} (-1)^{l+1} \left(\frac{1}{l}\left(\frac{\prior}{1-\prior}\right)^l\right)   \sum_{k \in [m]} \left(1-\exp(-l \bw_k)\right)_{S_a} \otimes \left(1-\exp(-l \bw_k)\right)_{S_b} \otimes \left(1-\exp(-l \bw_k)\right)_{S_c}  \label{eq:exprpmi}\end{align}
\normalsize

where the last equality holds by noting that 
$$ \sum_{t \geq l} \prior^t \frac{1}{t} {t \choose l} =  \frac{1}{l} \left(\frac{\prior}{1-\prior}\right)^l $$

The term corresponding to $t=1$ is easily seen to be 
$$ \frac{\prior}{1-\prior} \sum_{k \in [m]} F_{k,S_a} \otimes F_{k,S_b} \otimes F_{k,S_c} $$
therefore we to show the statement of the lemma, we only need bound the contribution of the terms with $\l \geq L$. 

Toward that, note that $\forall l,k \|1 - \exp(-l \bw_k)\| \leq n$. Hence, we have by Lemma \ref{lem:tensor_norm_bound}, 
 
$$ \left\|\sum_{k=1}^m \left(1-\exp(-l \bw_k)\right)_{S_a} \otimes \left(1-\exp(-l \bw_k)\right)_{S_b} \otimes \left(1-\exp(-l \bw_k)\right)_{S_c}\right\|_{\{1,2\},\{3\}} \leq (mn)^3 $$ 

Therefore, subadditivity of the $\{12\},\{3\}$ norm gives
\begin{align*} 
& \left\|\sum_{l=L}^{\infty} (-1)^{l+1} \left(\frac{\left(\frac{\prior}{1-\prior}\right)^l}{l}\right)   \sum_{k \in [m]} \left(1-\exp(-l \bw_k)\right)_{S_a} \otimes \left(1-\exp(-l \bw_k)\right)_{S_b} \otimes \left(1-\exp(-l \bw_k)\right)_{S_c}\right\|_{\{1,2\},\{3\}}  \\
& \leq (m n)^3 \sum_{l=L}^{\infty} \left(\frac{\left(\frac{\prior}{1-\prior}\right)^l}{l}\right) \leq \frac{(mn)^3}{L} \sum_{l=L}^{\infty} \left(\frac{\prior}{1-\prior}\right)^l = \frac{(mn)^3}{L} \frac{\left(\frac{\prior}{1-\prior}\right)^L}{1 - \left(\frac{\prior}{1-\prior}\right)^L} \\
\end{align*}

which gives us what we need. 
\end{proof} 

A completely analogous proof gives a similar expression for the PMI matrix: 

\begin{proposition} 

For any subsets $S_a, S_b$ of $[n]$, s.t. $S_a \cap S_b = \emptyset$, the restriction of the PMI matrix $\pmi_{S_a,S_b}$ satisfies, 
for any $L \geq 2$, 
\small
\begin{equation} \pmi_{S_a, S_b} = \frac{\prior}{1-\prior} \sum_{k \in [m]} F_{k,S_a} F_{k,S_b} ^{\top} + 
\sum_{l = 2}^L (-1)^{l+1} \left(\frac{1}{l}\left(\frac{\prior}{1-\prior}\right)^l\right) \sum_{k \in [m]} (\tilde{\taylor}_l)_{k,S_a} ((\tilde{\taylor}_l)_{k,S_b})^{\top} + E_L \end{equation} 
\normalsize
where 
$$\|E_L\|_{\{1,2\},\{3\}} \leq \frac{(mn)^2}{L} \frac{\left(\frac{\prior}{1-\prior}\right)^L}{1 - \left(\frac{\prior}{1-\prior}\right)^L}$$  

\end{proposition}

\section{Spectral properties of the random model} 
\label{s:random}

The goal of this section is to prove that the random model specified in Section~\ref{s:intro} satisfies the incoherence property \ref{d:incoherence} on the weight matrix and the
spectral boundedness property of the PMI tensor. (Recall, the former is required for the whitening algorithm, and the later for the tensor decomposition algorithm.)

Before delving into the proofs, we will need a few simple bounds on the singular values of $P_l$. 

\begin{lemma} Let $S_a \subseteq [n]$, s.t. $|S_a| = \Omega(n)$. With probability $1 - \exp(-\log^2 n)$ over the choice of $W$,  and for all $l = O(\mbox{poly}(n))$, 
$$\sigma_{\min}(\taylor_{l,S_a}) \gtrsim n \spr $$ 
and 
$$\sigma_{\max}(\taylor_{l,S_a}) \lesssim m n \spr^2 $$   
\label{l:boundevals} 
\end{lemma}
\begin{proof}
 
Let us proceed to the lower bound first. 

If we denote by $L$ the matrix which has as columns $(\tilde{\taylor}_l)_{k,S_a}$, $k \in [m]$, it's clear that $\taylor_{l,S_a} = LL^{\top}$. Since
$$\sigma_{\min} (LL^{\top}) = \sigma_{\min}(L^{\top}L)$$  
we will proceed to bound the smallest eigenvalue of $L^{\top} L$. 

Note that 
$$L^{\top}L = \sum_{k \in S_a} (1-\exp(-l W^k)) (1-\exp(-l W^k))^{\top} $$ 

Since the matrices $(1- \exp(-l W^k)) (1-\exp(-lW^k))^{\top}$ are independent, the bound will follow from a matrix Bernstein bound. Denoting 
$$Q = \E\left[ (1- \exp(-l W^k)) (1-\exp(-l W^k))^{\top}\right]$$ 
by a simple calculation we have

\begin{equation} Q  = \spr^2 \E\left[1 - \exp(-l\tilde{W})\right]^2 \mathbf{1}\mathbf{1}^{\top} + \left(\spr \E\left[(1 - \exp(-l\tilde{W}))^2\right]-\spr^2\E\left[1 - \exp(-l\tilde{w})\right]^2\right) \Id_m \label{eq:exprq}\end{equation} 
where $\mathbf{1}$ is the all-ones vector of size $m$, and $\Id_m$ is the identity of the same size. Furthermore, $\tilde{W}$ is a random variable following the distribution $\mathcal{D}$ of all the $\tilde{W}_{i,j}$.

Note that \eqref{eqn:tildeWij} together with the assumption $\nu = \Omega(1)$ gives $\sigma_{\min}(Q) = \Omega(\spr)$ 

Let $Z^i = Q^{-1/2} (1- \exp(-l W_i))$. Then we have that $ \Exp\left[\sum_{i \in S_a} Z^i (Z^i)^{\top}\right] = |S_a| \cdot\Id_m$, and with high probability, $\|Z^i\|^2 \le 1/\sigma_{\min}(Q)\cdot \norm{F^i}^2\lesssim m $ and it's a sub-exponential random variable. Moreover, $$r^2 = \Norm{\sum_i\Exp\left[\left(\sum Z^i (Z^i)^{\top}\right)^2\right] } \lesssim m \Norm{\E\left[\sum_i Z^i(Z^i)^{\top}\right]}\le mn \mper$$

Therefore, by Bernstein inequality we have that w.h.p, 
\begin{align}
\Norm{\sum_{i \in S_a} Z^i(Z^i)^{\top} - n\Id_m}\lesssim \sqrt{r^2\log n} + \max \|Z^i\|^2 \log n = \sqrt{mn\log n} + m  \log n\mper \nonumber
\end{align}

It follows that\begin{align}
\sum_{i \in S_a} Z^i(Z^i)^{\top}\succeq \left(n- O\left( \sqrt{mn\log n} \right)\right)\Id_m \succeq n \Id_m\mper\nonumber
\end{align}
which in turn implies that $\taylor_{l,S_a} \succeq n Q$. But this immediately implies $\sigma_{\min}(\taylor_{l,S_a}) \gtrsim n \spr$ with high probability. Union bounding over all $l$, we get the first part of the lemma. 

The upper bound will be proven by a Chernoff bound. 
Note that the matrices $$(1- \exp(-l W_k))_{S_a} (1-\exp(-lW_k))_{S_a}^{\top}, k \in [m]$$ are independent. Furthermore, $\|(1- \exp(-l W_k))_{S_a} (1-\exp(-lW_k))_{S_a}^{\top}\|^2 \leq p n$ with high probability, and the variable $\|(1- \exp(-l W_k))_{S_a} (1-\exp(-lW_k))_{S_a}^{\top}\|^2 $ is sub-exponential. Finally, 
\begin{align*} r^2 & = \left\| \E\left[\sum_{k=1}^m ((1- \exp(-l W_k))_{S_a} (1-\exp(-lW_k))_{S_a}^{\top})^2\right] \right\| \\
& \leq \sum_{k=1}^m \left\| \E\left[((1- \exp(-l W_k)_{S_a}) (1-\exp(-lW_k))_{S_a}^{\top})^2\right] \right\| \\ 
& \leq m \left\|1- \exp(-l W_k)_{S_a}\right\|^2 \E\left[((1- \exp(-l W_k))_{S_a} (1-\exp(-l W_k))_{S_a}^{\top})\right] \leq m n^2 p^2\end{align*}
Similarly as in the lower bound, 
$$\E[\taylor_{l,S_a}] =  \spr^2 \E\left[1 - \exp(-l\tilde{W})\right]^2 \mathbf{1}\mathbf{1}^{\top} + \left(\spr \E\left[(1 - \exp(-l\tilde{W}))^2\right]-\spr^2\E\left[1 - \exp(-l\tilde{w})\right]^2\right) \Id_{|S_a|}$$ 
\sloppy where $\mathbf{1}$ is the all-ones vector of size $|S_a|$, and $\Id_{|S_a|}$ is the identity of the same size. Again, $\tilde{W}$ is a random variable following the distribution $\mathcal{D}$ of all the $\tilde{W}_{i,j}$. 
This immediately gives
$$ \taylor_l \preceq \E[\taylor_l] + r \log n \Id_{|S_a|} \preceq m n \spr^2 + \sqrt{m n^2 \spr^2} \log n \Id_{|S_a|} \preceq O( m n \spr^2 ) \Id_{|S_a|}$$ 

A union bound over all values of $l$ gives the statement of the Lemma. 

\end{proof} 

\subsection{Incoherence of matrix $F$}	
	
First, we proceed to show the incoherence property \ref{d:incoherence} on the weight matrix.
\begin{lemma} Suppose $n$ is a multiple of 3. 
	Let $\bfr = \svdU \Sigma \svdV$ be the singular value decomposition of $\bfr$.  
	Let $S_a,S_b,S_c$ be  a uniformly random equipartition of the rows of $[n]$. Suppose $F$ is $\mu$-incoherent with $\mu \le n/(m\log n)$.  Then, with high probability over the choice of  and $S_a,S_b,S_c$, we have for every $i \in \{a,b,c\}$,  
	$$ \Norm{\left(\svdU^{\subsvd_i}\right)^{\top} \svdU^{\subsvd_i} - \frac{1}{3}\Id } \lesssim \sqrt{\frac{\mu m}{n} \log n}\mper$$
	\label{lem:incoherentconc}
\end{lemma}  
\begin{proof} 
Let $S = S_a$. 
Then, since $U^{\top}U = \Id_m$, we have  
$$\E\left[\sum_{i \in S}  (U^i) (U^i)^{\top}\right] = \frac{1}{3} \cdot \Id_m\mper$$ By the assumption on the row norms of $U$, $\|\svdU^i (\svdU^i)^{\top}\|_2 = \|\svdU^i\|^2 \leq \mu \frac{m}{n}$. 
By the incoherence assumption, we have that $\max_i \norm{U^i}^2\le \mu m/n$.  

We also note that $U^i$'s are negatively associated random variables. Therefore by the matrix Chernoff inequality for negatively associated random variables, we have with high probability,
\begin{align}
\left\|(\svdU^{S})^{\top}\svdU^{S}-\E\left[(\svdU^{S})^{\top}\svdU^{S}\right]\right\| \lesssim \sqrt{\frac{\mu m\log n}{n}}\mper\nonumber
\end{align}

But, an analogous argument holds for $S_b, S_c$ as well -- so by a union bound over $k$, we complete the proof. 
\end{proof}

\begin{lemma} 
	Suppose $n \gtrsim m\log n$. 
	Under the generative assumption in Section~\ref{s:intro} for $W$, we have that 
	we have that $\bfr = 1-\exp(-W)$ is $O(1)$-incoherent. 
\end{lemma} 
\label{l:rowbound}
\begin{proof}
	
We have that  $\bfr\bfr^{\top} = \svdU \Sigma^2 \svdU^T$ and therefore,  $\|\bfr^i\|^2 =  \Sigma^2_{i,i} \|\svdU^i\|^2$. This in turn implies that \begin{align}
\|\svdU^i\|^2 \leq \frac{1}{\min \Sigma_{ii}^2}\Norm{F^i} = \frac{1}{\sigma^2_{\min}(\bfr)} \|\bfr^i\|^2 \nonumber\mper
\end{align}
Since $\Norm{F^i}^2 \le \spr m + \sqrt{\spr m} \leq 2 \spr m $ with high probability,  
we only need to bound $\sigma_{\min}(\bfr)$ from below. 
Note that 
$
\sigma^2_{\min}(\bfr) = \sigma_{\min}(\bfr^{\top}\bfr) $. Therefore it suffices to control $\sigma_{\min}(F^{\top}F)$. 
But by Lemma~\ref{l:boundevals} we have  $\sigma^2_{\min}(\bfr) \gtrsim n \spr$.

Therefore, we have that 
\begin{align}
\|\svdU^i\|^2 \leq \frac{1}{\sigma^2_{\min}(\bfr)} \|\bfr^i\|^2 = O( \frac{m}{n}) \nonumber
\end{align}

\end{proof}

\subsection{Spectral boundedness}

The main goal of the section is to show that the bias terms in the PMI tensor are spectrally bounded by the PMI matrix (which we can estimate from samples). 
Furthermore, we show that we can calculate an approximate whitening matrix for the leading terms of the PMI tensor using the machinery in Section \label{s:whitening}.

The main proposition we will show is the following:

\begin{proposition}
Let $W$ be sampled according to the random model in Section \ref{s:intro}  with $\rho = o\left(\frac{1}{\log n}\right), \spr = \omega\left(\frac{\log n}{ \sqrt{m} n}\right)$. 
Let $S_a \subseteq [n]$, $|S_a| = \Omega(n)$. If $R_{S_a}$ is the matrix that has as columns the vectors $\left(\frac{1}{l}\left(\frac{\prior}{1-\prior}\right)^l\right)^{1/3} (\tilde{\taylor}_l)_{j,S_a}$, $l \in [2,L], j \in [m]$ for $L = O(\mbox{poly}(n))$, and $A$ is the matrix that has as columns the vectors $\left(\frac{\prior}{1-\prior}\right)^{1/3} (\tilde{\taylor}_1)_{j,S_a}$ for $j \in [m]$, with high probability it holds that
$R_{S_a}R_{S_a}^{\top}$ is $O(\prior^{2/3} \log n)$-spectrally bounded by $A$. 

\label{p:thirdspectral}
\end{proposition}

The main element for the proposition is the following Lemma: 

\begin{lemma} 
	For any set $S_a \subseteq [n]$, $|S_a| = \Omega(n)$, with high probability over the choice of $\weight$, for every $\ell, \ell' = O(\mbox{poly}(n))$ 
	$\taylor_{l,S_a}\taylor_{l,S_a}^{\top}$ is $O(\log n)$-spectrally bounded by $\taylor_{l',S_a}$. 
\label{l:spectralbound_taylor}
\end{lemma}

Before proving the Lemma, let us see how the proposition follows from it:

\begin{proof}[Proof of \ref{p:thirdspectral}]

We have 
\begin{align*}  R_{S_a} R_{S_a}^{\top} = \sum_{l = 2}^{L}  \left(\frac{1}{l} \left(\frac{\prior}{1-\prior}\right)^l\right)^{2/3} \taylor_{l,S_a}  \end{align*}

By Lemma \ref{l:spectralbound_taylor}, we have that $\forall l > 1$, $\tilde{\taylor}_{l,S_a}$ is $\tau$-spectrally bounded by $\tilde{\taylor}_{1,S_a}$, for some $\tau = O( \log n)$. 
Hence, 

\begin{align*}
 R_{S_a} R_{S_a}^{\top} & \preceq  \sum_{l = 2}^{\infty}  \left(\frac{1}{l} \left(\frac{\prior}{1-\prior}\right)^l\right)^{2/3} \tau( \taylor_{1,S_a} + \sigma_{\min}(\taylor_{1,S_a})) \\ 
\\& \precsim \rho^{4/3} \tau (\taylor_{1,S_a} + \sigma_{\min}(\taylor_{1,S_a}))
\end{align*} 

Since $AA^{\top} = (\frac{\rho}{1-\rho})^{2/3} \taylor_{1,S_a}$, the claim of the Proposition follows. 

It is clear analogous statements hold for $S_b$ and $S_c$. 
\end{proof} 

Finally, we proceed to show Lemma \ref{l:spectralbound_taylor}.

For notational convenience, we will denote by $J_{m \times n}$ the all ones matrix with dimension $m\times n$. (We will omit the dimensions when clear from the context.)

This statement will immediately follow from the following two lemmas: 

\begin{lemma} For any set $S_a \subseteq [n]$, $|S_a| = \Omega(n)$, with probability $1 - \exp(-\log^2 n)$ over the choice of $\weight$, for all $\ell \leq O(\mbox{poly}(n))$, 
$$ \taylor_{l,S_a} \preceq 10 n \spr \log n \Id + \frac{5}{2} m \spr^2 J $$ 
\label{l:spectralupper} 
\end{lemma} 

\begin{lemma} For any set $S_a \subseteq [n]$, $|S_a| = \Omega(n)$, with probability $1 - \exp(-\log^2n)$ over the choice of $\weight$, $\forall \ell = \mbox{poly}(n)$,  
$$ \taylor_{l,S_a} + 6 n \spr \log n \Id \precsim m \spr^2 J  $$ 
\label{l:spectrallower} 
\end{lemma} 

Before showing these lemmas, let us see how Lemma~\ref{l:spectralbound_taylor} is implied by them: 
\begin{proof}[Proof of Lemma \ref{l:spectralbound_taylor}]

Let $\kappa$ be the constant in ~\ref{l:spectrallower}, s.t. $\taylor_{l,S_a} + 6 n \spr \log n \Id \precsim m \spr^2 J$. 
Putting the bounds from Lemmas \ref{l:spectralupper} and \ref{l:spectrallower} together along with a union bound, we have that with high probability, $\forall l,l' = O(\mbox{poly}(n))$
$$ \taylor_{l,S_a} - \frac{5}{2} \kappa \taylor_{l',S_a} \preceq \left(10n \spr \log n + \frac{15}{2} \kappa  n \spr \log n\right) \Id \preceq O(m \spr \log n) \Id$$
 
But, note that $\sigma_{\min}(\taylor_{l'}) = \Omega(n \spr)$, by Lemma~\ref{l:boundevals}.
Hence, $\taylor_l - \frac{5}{2} \kappa \taylor_{l',S_a} \preceq r \log n \sigma_{\min}(\taylor_{l'})$, for some sufficiently large constant $r$. This implies 
$$\taylor_l - r \log n \taylor_{l',S_a} \preceq \taylor_l - \frac{5}{2} \kappa \taylor_{l'} \preceq  r \log n \sigma_{\min}(\taylor_{l',S_a}) $$ 
from which the statement of the lemma follows.   

\end{proof}

We proceed to the first lemma: 
\begin{proof}[Proof of Lemma \ref{l:spectralupper}]
To make the notation less cluttered, we will drop $l$ and $S_a$ and use $\taylor = \taylor_{l,S_a}$. 
Furthermore, we will drop $S_a$ when referring to columns of $\tilde{\taylor}$ so we will denote $\tilde{\taylor}_k = \tilde{\taylor}_{k,S_a}$. 
           
Let's denote by $e = \frac{1}{\sqrt{|S_a|}} \mathbf{1}$ . Let's furthermore denote $\Id_{\mathbf{1}} = ee^{\top}$, and $\Id_{-\mathbf{1}} = \Id - ee^{\top}$. Note first that trivially, since $\Id_{\mathbf{1}} + \Id_{-\mathbf{1}} = \Id$,
\begin{equation} \taylor = ( \Id_{\mathbf{1}} +  \Id_{-\mathbf{1}}) \taylor ( \Id_{\mathbf{1}} + \Id_{-\mathbf{1}}) \label{eq:idrandso}\end{equation}
Furthermore, it also holds that  
\begin{align*} 0 & \preceq ( 2 \Id_{-\mathbf{1}} - \frac{1}{2} \Id_{\mathbf{1}}) \taylor (2 \Id_{-\mathbf{1}} - \frac{1}{2} \Id_{\mathbf{1}}) \\ 
                & = \frac{1}{4} \Id_{\mathbf{1}} \taylor \Id_{\mathbf{1}} + 4 \Id_{-\mathbf{1}} \taylor \Id_{-\mathbf{1}} - \Id_{-\mathbf{1}} \taylor \Id_{\mathbf{1}} - \Id_{\mathbf{1}} \taylor \Id_{-\mathbf{1}}   
\end{align*} 
 where the first inequality holds since 
$$ ( 2 \Id_{-\mathbf{1}} - \frac{1}{2} \Id_{\mathbf{1}}) \taylor (2 \Id_{-\mathbf{1}} - \frac{1}{2} \Id_{\mathbf{1}}) = 
\left(( 2 \Id_{-\mathbf{1}} - \frac{1}{2} \Id_{\mathbf{1}}) \tilde{\taylor}\right)\left(( 2 \Id_{-\mathbf{1}} - \frac{1}{2} \Id_{\mathbf{1}}) \tilde{\taylor}\right)^{\top} $$

From this we get that 
\begin{equation} \Id_{-\mathbf{1}} \taylor \Id_{\mathbf{1}} + \Id_{\mathbf{1}} \taylor \Id_{-\mathbf{1}} \preceq \frac{1}{4} \Id_{\mathbf{1}} \taylor \Id_{\mathbf{1}} + 4 \Id_{-\mathbf{1}} \taylor \Id_{-\mathbf{1}}  \label{eq:idrand2so}								
\end{equation}

We proceed to upper bound both the terms on the RHS above. More precisely, we will show that 
\begin{equation} \Id_{\mathbf{1}} \taylor \Id_{\mathbf{1}} \leq 2 m \spr^2 n J \label{eq:id1term} \end{equation} 
\begin{equation} \Id_{-\mathbf{1}} \taylor \Id_{-\mathbf{1}} \preceq 2 n \spr \log n \Id \label{eq:id2term}  \end{equation} 

Let us proceed to showing \eqref{eq:id1term}. The LHS can be rewritten as 
$$  \Id_{\mathbf{1}} \taylor \Id_{\mathbf{1}} = ee^{\top} \left(e^{\top} \tilde{\taylor} \tilde{\taylor}^{\top} e \right)  $$
Note that 
$$ e^{\top} \tilde{\taylor} \tilde{\taylor}^{\top} e = \frac{1}{n} \left( \sum_{k=1}^m \langle  \mathbf{1}, \tilde{\taylor}_k  \rangle^2  \right)$$ 
All the terms $\langle \mathbf{1}, \tilde{\taylor}_k \rangle^2$ are independent and satisfy
and $\E[\langle \mathbf{1}, \tilde{\taylor}_k \rangle^2] \leq (\E[\langle \mathbf{1}, \tilde{\taylor}_k \rangle ])^2 \leq \spr^2 n^2$. 
By Chernoff, we have that 
$$ \sum_{k=1}^m \langle  \mathbf{1}, \tilde{\taylor}_k  \rangle^2 \leq m \spr^2 n^2 + \sqrt{m \spr^2 n^2} \log n \leq 2 m \spr^2 n^2 $$
with probability at least $1 - \exp(-\log^2 n)$, where the second inequality holds because $\spr = \omega(\frac{\log n}{ \sqrt{m} n})$. Hence, $e^{\top} \tilde{\taylor} \tilde{\taylor}^{\top} e  \leq 2 m \spr^2 n$ with high probability.

We proceed to \eqref{eq:id2term}, which will be shown by a Bernstein bound. Towards that, note that

\begin{align*} \E[(\Id_{-\mathbf{1}}\tilde{\taylor}_k (\Id_{-\mathbf{1}}\tilde{\taylor}_k)^{\top}] & = \Id_{-\mathbf{1}} \E[\tilde{\taylor}_k (\tilde{\taylor}_k)^{\top}] \Id_{-\mathbf{1}} \\ 
& = \Id_{-\mathbf{1}} \left(\spr^2 \E[1-\exp(\tilde{W})]^2\mathbf{1}\mathbf{1}^{\top}  + (\spr - \spr^2) \E[(1-\exp(\tilde{W}))^2] \Id \right)\Id_{-\mathbf{1}} \\ 
& \preceq \spr \Id  
\end{align*}
where $\tilde{W}$ is a random variable following the distribution $\mathcal{D}$ of all the $\tilde{W}_{i,j}$. 
The second line can be seen to follow from the independence of the coordinates of $\tilde{\taylor}_k$ according to our model.

Furthermore, with high probability $\|\Id_{-\mathbf{1}}\tilde{\taylor}_k\|^2_2 \leq \|\tilde{\taylor}_k\|^2_2 \leq n \spr$ and the random variable $\|\Id_{-\mathbf{1}}\tilde{\taylor}_k\|^2$ is sub-exponential
 Finally,  
\begin{align*} r^2 = \left\|\E\left[\sum_{k=1}^m \left((\Id_{-\mathbf{1}}\tilde{\taylor}_k) (\Id_{-\mathbf{1}}\so_k)^{\tilde{\taylor}}\right)^2\right] \right\|  & \leq 
\left\|\tilde{\taylor}_k\right\|^2_2 \left\|\E\left[\sum_{k=1}^m (I_{-\mathbf{1}}\tilde{\taylor}_k) (\Id_{-\mathbf{1}}\tilde{\taylor}_k)^{\top}\right] \right\| \\ & \leq n \spr m \|\E[(\Id_{-\mathbf{1}}\tilde{\taylor}_k) (\Id_{-\mathbf{1}}\tilde{\taylor}_k)^{\top}]\| \leq n \spr^2 m \end{align*} 
Therefore, applying a matrix Bernstein bound, we get 
$$ \Id_{-\mathbf{1}} \taylor \Id_{-\mathbf{1}} \preceq m \spr \Id + n \spr \log n \Id + r \log n \Id \preceq 2 n \spr \log n \Id $$ 
with high probability. 

Combining this with \eqref{eq:idrandso} and \eqref{eq:idrand2so}, we get 
\begin{align*} \taylor & \preceq \frac{5}{4} \Id_{\mathbf{1}} \taylor \Id_{\mathbf{1}} + 5 \Id_{-\mathbf{1}} \taylor \Id_{-\mathbf{1}} \\
                              & \preceq \frac{5}{2} m \spr^2 J + 10 n \spr \log n \Id
\end{align*}

\end{proof} 

Let us proceed to the second inequality, which essentially follows the same strategy: 

\begin{proof}[Proof of Lemma \ref{l:spectrallower}]
Similarly as in the proof of Lemma \ref{l:spectralupper}, for notational convenience, let's denote  by $\tilde{\taylor}$ the matrix which has column $k$ the vector $\left(1-\exp(l \bw_k)\right)$. 

Reusing the notation from Lemma \ref{l:spectrallower}, we have that 
\begin{equation} \taylor = ( \Id_{\mathbf{1}} +  \Id_{-\mathbf{1}}) \taylor ( \Id_{\mathbf{1}} + \Id_{-\mathbf{1}}) \label{eq:idrandfo}\end{equation}
and 
\begin{align*} 0 & \preceq (\frac{1}{2} \Id_{\mathbf{1}} + 2 \Id_{-\mathbf{1}}) \taylor (\frac{1}{2} \Id_{\mathbf{1}} + 2 \Id_{-\mathbf{1}}) \\ 
                & = \Id_{-\mathbf{1}} \taylor' \Id_{\mathbf{1}} + \Id_{\mathbf{1}} \taylor \Id_{-\mathbf{1}} + \frac{1}{4} \Id_{\mathbf{1}} \taylor \Id_{\mathbf{1}} + 4 \Id_{-\mathbf{1}} \taylor \Id_{-\mathbf{1}} 
\end{align*} 
for similar reasons as before. 
From this we get that 
\begin{equation} \Id_{-\mathbf{1}} \taylor \Id_{\mathbf{1}} + \Id_{\mathbf{1}} \taylor \Id_{-\mathbf{1}} \succeq -\frac{1}{4} \Id_{\mathbf{1}} \taylor \Id_{\mathbf{1}} - 4 \Id_{-\mathbf{1}} \taylor \Id_{-\mathbf{1}}  \label{eq:idrand2fo}								
\end{equation}
 Putting \eqref{eq:idrandfo} and \eqref{eq:idrand2fo} together, we get that 
\begin{equation} \taylor + 3 \Id_{-\mathbf{1}} \taylor \Id_{-\mathbf{1}} \succeq \frac{3}{4} \Id_{\mathbf{1}} \taylor \Id_{\mathbf{1}} \label{eq:lower1fo} \end{equation}

We will proceed to show an upper bound $\Id_{-\mathbf{1}} \taylor \Id_{-\mathbf{1}} \preceq 2 n  \spr \log n\Id$ on second term of the LHS.  
We will do this by a Bernstein bound as before. Namely, analogously as in Lemma \ref{l:spectrallower}, $$ \Id_{-\mathbf{1}} \taylor \Id_{-\mathbf{1}}  = \sum_{k=1}^m (\Id_{-\mathbf{1}}\tilde{\taylor}_k) (\Id_{-\mathbf{1}}\tilde{\taylor}_k)^{\top}  $$
and $\E[(\Id_{-\mathbf{1}}\tilde{\taylor}_k) (\Id_{-\mathbf{1}}\tilde{\taylor}_k)^{\top}] \preceq \spr \Id$ and 
and $\|\Id_{-\mathbf{1}}\tilde{\taylor}_k\|^2_2 \leq \|\tilde{\taylor}_k\|^2_2 \leq n \spr$ are satisfied so  
\begin{align*} r^2 = \left\|\E\left[\sum_{k=1}^m \left((\Id_{-\mathbf{1}}\tilde{\taylor}_k) (\Id_{-\mathbf{1}}\tilde{\taylor}_k)^{\top}\right)^2\right] \right\| \leq n \spr^2 m \end{align*} 
Therefore, applying a matrix Bernstein bound, we get 
$$ \Id_{-\mathbf{1}} \taylor \Id_{-\mathbf{1}} \preceq m \spr \Id + n \spr \log n \Id + r \preceq 2 n \spr \log n \Id $$ 
with high probability. 

Plugging this in in \eqref{eq:lower1fo}, we get 
$$\taylor  + 6 n \spr \log n \Id \succeq \frac{3}{4} \Id_{\mathbf{1}} \taylor \Id_{\mathbf{1}} = \frac{3}{4} ee^{\top} \tilde{\taylor} \tilde{\taylor}^{\top} ee^{\top} = \frac{3}{4} ee^{\top} \left(e^{\top} \tilde{\taylor} \tilde{\taylor}^{\top} e \right) $$  
Since we have
$ e^{\top} \tilde{\taylor} \tilde{\taylor}^{\top} e = \frac{1}{n} \left( \sum_{k=1}^m \langle \tilde{\taylor}_k , \mathbf{1} \rangle^2  \right)$ 
with the goal of applying Chernoff, we will lower bound $\E[\langle \mathbf{1}, \fo_k \rangle^2]$. More precisely, we will show 
$\E[\langle \mathbf{1}, \tilde{\taylor}_k \rangle]^2 = \Omega (n^2 \spr^2) $. 
In order to do this, we have 
\begin{align*}
\E[\langle \mathbf{1}, \tilde{\taylor}_k \rangle^2] & = \sum_{j \in S_a} \E[(\tilde{\taylor}_k)^2_j] + \sum_{j \neq j'; j,j' \in S_a} \E[(\tilde{\taylor}_k)_j]\E[(\tilde{\taylor}_k)_{j'}] \\
& = \sum_{j \in S_a} \spr \E[(1 - \exp(-\tilde{lW})^2] +  \sum_{j \neq j'; j,j' \in S_a} \spr^2 \E[1 - \exp(-\tilde{lW})]^2 \\
& \geq \sum_{j \in S_a} \spr \E[(1 - \exp(-\tilde{W})^2] +  \sum_{j \neq j'; j,j' \in S_a} \spr^2 \E[1 - \exp(-\tilde{W})]^2 \\
& = \Omega (n^2 \spr^2) \end{align*}
where $\tilde{W}$ is a random variable following the distribution $\mathcal{D}$ of all the $\tilde{W}_{i,j}$. 
and the last inequality holds because of \eqref{eqn:tildeWij}. 

 So by Chernoff, we get that 
$e^{\top} \tilde{\taylor} \tilde{\taylor}^{\top} e  = \frac{1}{n} (\Omega( m n^2 \spr^2 ) - \sqrt{n^2 \spr^2 m}) = \Omega(m n \spr^2)$ 
with high probability.
Altogether, this means 
$$\taylor   + 6 n \spr \log n \Id   \succsim m \spr^2 J $$

\end{proof}

\section{Robust whitening}
\label{s:whitening}

\begin{algorithm}\caption{Obtaining whitening matrices}\label{a:alg1} 
	{\bf Inputs: } Random partitioning $S_a, S_b, S_c$ of $[n]$. Empirical PMI matrix $\widehat{\pmi}$. 
	
	{\bf Outputs: } Whitening matrices $Q_a, Q_b, Q_c\in \R^{d\times d}$
	\begin{enumerate}
		\item Output 
		\begin{align*} & Q_a = \rho^{-1/3} \widehat{\pmi}_{S_a,S_b} (\widehat{\pmi}^{+}_{S_b,S_c})^{\top} \widehat{\pmi}_{S_c,S_a}, \\
		& Q_b = \rho^{-1/3} \widehat{\pmi}_{S_b,S_c} (\widehat{\pmi}^{+}_{S_c,S_a})^{\top} \widehat{\pmi}_{S_a,S_b}, \\
		& Q_c = \rho^{-1/3} \widehat{\pmi}_{S_c,S_a} (\widehat{\pmi}^{+}_{S_a,S_b})^{\top} \widehat{\pmi}_{S_b,S_c}\end{align*}
	\end{enumerate}
\end{algorithm}
In this section, we show the formula $Q_a = \rho^{-1/3} \widehat{\pmi}_{S_a,S_b} (\widehat{\pmi}^{+}_{S_b,S_c})^{\top} \widehat{\pmi}_{S_c,S_a}$ computes an approximation of the true whitening matrix $AA^\top$, so that the error is $\epsilon$-spectrally bounded by $A$. We recall Theorem~\ref{thm:whiten}. 
\begin{reptheorem}{thm:whiten}
	Let $n\ge m$ and $A,B,C\in \R^{n\times m}$. Suppose $\Sigma_{ab}, \Sigma_{bc},\Sigma_{ca}\in \R^{n\times n}$ are of the form, 
	$$
	\Sigma_{ab} = AB^\top + E_{ab},~~\Sigma_{bc} = BC^\top + E_{bc},~~\textup{and}~~\Sigma_{ca} = CA^\top + E_{ca}.
	$$
	where $E_{ab}, E_{bc}, E_{ca}$ are $\epsilon$-spectrally bounded by $(A,B)$, $(B,C)$, $(C,A)$ respectively.
	Then,  the matrix matrix  
	$$Q_a= \Sigma_{ab}[\Sigma_{bc}^\top]_m^+\Sigma_{ca}$$
	is a good approximation of $AA^{\top}$ in the sense that $Q_a = \Sigma_{ab}[\Sigma_{bc}^\top]_m^+\Sigma_{ca}-AA^{\top}$ is $O(\epsilon)$-\spb by $A$. Here $[\Sigma]_m$ denotes the best rank-$m$ approximation of $\Sigma$. 

\end{reptheorem}

Towards proving Theorem~\ref{thm:whiten}, an intermediate step is to understand the how the space of singular vectors of $BC^{\top}$ are aligned with the noisy version $\Sigma_{bc}$. The following explicitly represent $BC^{\top} +E$ as the form $B'R(C')^{\top} + \Delta'$. Here the crucial benefit to do so is that the resulting $\Delta'$ is small in every direction. In other words, we started with a relative error guarantees on $E$ and the Lemma below converts to it an absolute error guarantees on $\Delta'$ (though the signal term changes slightly).

\begin{lemma}\label{lem:represent}
Suppose $B,C$ are $n\times m$ matrices with $n\ge m$. Suppose a matrix $E$ is $\epsilon$-spectrally bounded by $(B,C)$, then $BC^\top+E$ can be written as
$$
BC^\top+E = (B+\Delta_B)R_{BC}(C+\Delta_C)^\top + \Delta_{BC}'\,, 
$$
where $\Delta_B,\Delta_C,\Delta_{BC}'$ are small and $R_{BC}$ is close to identity in the sense that, 
\begin{align}
\|\Delta_B\| & \le O(\epsilon \sigma_{min}(B))\nonumber\\
\|\Delta_C\| &\le O(\epsilon\sigma_{min}(C))\nonumber\\
\|\Delta_{BC}'\| &\le O(\epsilon \sigma_{min}(B)\sigma_{min}(C))\nonumber\\
\|R_{BC}-\Id\|& \le O(\epsilon)\nonumber
\end{align}
\end{lemma}

\begin{proof}
The key intuition is if the perturbation is happening in the span of columns of $B$ and $C$, they cannot change the subspace. By Definition~\ref{def:asymmetric}, we can write $E$ as

$$
E = B\Delta_1 C^\top + B\Delta_2^\top + \Delta_3 C^\top + \Delta_4.
$$

Now since $\|\Delta_1\| \le \epsilon < 1$, we know $(\Id-\Delta_1)$ is invertible, so we can write

$$
(BC^\top + E) = (B+\Delta_3(\Id+\Delta_1)^{-1})(\Id+\Delta_1) (C + \Delta_2(\Id-\Delta_1)^{-\top})^\top - \Delta_3(\Id-\Delta_1)^{-1}\Delta_2^\top + \Delta_4.
$$

This is already in the desired form as we can let $\Delta_B = \Delta_3(\Id+\Delta_1)^{-1}$, $R_{BC} = (\Id+\Delta_1)$, $\Delta_C = \Delta_2(\Id-\Delta_1)^{-\top}$, and $\Delta_{BC}' = - \Delta_3(\Id-\Delta_1)^{-1}\Delta_2^\top + \Delta_4$. By Weyl's Theorem we know $\sigma_{min}(\Id+\Delta_1) \ge 1-\epsilon$, therefore $\|\Delta_B\| \le \|\Delta_3\|\sigma_{min}^{-1}(\Id+\Delta) \le \frac{\epsilon}{1-\epsilon} \sigma_{min}(B)$. Other terms can be bounded similarly.

Now we prove that the top $m$ approximation of $BC^\top+E$ has similar column/row spaces as $BC^\top$. Let $U_B$ be the column span of $B$, $U_B'$ be the column span of $(B+\Delta_B)$, and $U_B''$ be the top $m$ left singular subspace of $(BC^\top+E)$. Similarly we can define $U_C$, $U_C'$, $U_C''$ to be the column spans of $C$, $C+\Delta_C$ and the top $m$ right singular subspace of $(BC^\top+E)$.

For $B+\Delta_B$, we can apply Weyl's Theorem and Wedin's Theorem. By Weyl's Theorem we know $\sigma_{min}(B+\Delta_B) \ge \sigma_{min}(B)-\|\Delta_B\| \ge (1-O(\epsilon))\sigma_{min}(B)$. By Wedin's Theorem we know $U_B'$ is $O(\epsilon)$-close to $U_B$. Similar results apply to $C+\Delta_C$.

Now we know $\sigma_{min}((B+\Delta_B)R_{BC}(C+\Delta_C)^\top)) \ge \sigma_{min}(B+\Delta_B)\sigma_{min}(R_{BC})\sigma_{min}(C+\Delta_C) \ge \Omega(\sigma_{min}(B)\sigma_{min}(C)$. Therefore we can again apply Wedin's Theorem, considering $(B+\Delta_B)R_{BC}(C+\Delta_C)^\top)$ as the original matrix and $\Delta'_{BC}$ as the perturbation. As a result, we know $U_B''$ is $O(\epsilon)$ close to $U_B'$, $U_C''$ is $O(\epsilon)$ close to $U_B'$. The distance between $U_B, U_B''$ (and $U_C,U_C''$) then follows from triangle inequality.

\end{proof}

As a direct corollary of Lemma~\ref{lem:represent}, we obtain that the $BC^{\top}$ and $BC^{\top} +E$ have similar subspaces of singular vectors. 
\begin{corollary}
	In the setting of Lemma~\ref{lem:represent},  let $[BC^\top + E]_m$ be the best rank-$m$ approximation of $BC^\top + E$. Then, the span of columns of $[BC^\top + E]_m$ is $O(\epsilon)$-close to the span of columns of $B$, span of rows of $[BC^\top + E]_m$ is $O(\epsilon)$-close to the span of columns of $C$.
	
	 Furthermore,  we can write 
	$	[BC^\top+E]_m = (B+\Delta_B)R_{BC}(C+\Delta_C)^\top + \Delta_{BC}.$
	Here $\Delta_B$, $\Delta_C$ and $R_{BC}$ as defined in Lemma~\ref{lem:represent} and $\Delta_{BC}$ satisfies $\|\Delta_{BC}\| \le O(\epsilon \sigma_{min}(B)\sigma_{min}(C))$.
\end{corollary}

\begin{proof}
Since $[BC^\top +E]_m$ is the best rank-$m$ approximation, because $(B+\Delta_B)R_{BC}(C+\Delta_C)^\top)$ is a rank $m$ matrix, in particular we have
	$$
	\|BC^\top +E - [BC^\top +E]_m\| \le \|BC^\top - (B+\Delta_B)R_{BC}(C+\Delta_C)^\top)\| \|\Delta'_{BC}\|.
	$$
	
	Therefore
	\begin{align*}
	\|\Delta_{BC}\| & = \|[BC^\top +E]_m - (B+\Delta_B)R_{BC}(C+\Delta_C)^\top)\| \\ &\le \|BC^\top +E - [BC^\top +E]_m\|+\|BC^\top + E - (B+\Delta_B)R_{BC}(C+\Delta_C)^\top)\| \\ &\le 2\|\Delta'_{BC}\|.
	\end{align*}
	
\end{proof}

In order to fix this problem, we notice that the matrix $[\Sigma_{bc}^\top]_m^+$ is multiplied by $\Sigma_{ab}$ on the left and $\Sigma_{ca}$ on the right. Assuming $\Sigma_{ab} = AB^\top$, $\Sigma_{ca} = CA^\top$, we should expect $[\Sigma_{bc}^{\top}]_m^+$ to ``cancel'' with the $B^\top$ factor on the left and the $C$ factor on the right, giving us $AA^\top$. Therefore, we should really measure the error of the middle term $[\Sigma_{bc}^{\top}]_m^+$ after left multiplying with $B^\top$ and right multiplying with $C$. We formalize this in the following lemma:
\begin{lemma}\label{lem:mainperturbterm}
Suppose $\Sigma_{bc}$ is as defined in Theorem~\ref{thm:whiten}, let $\Delta = [\Sigma_{bc}^{\top}]_m^+ - [CB^\top]^+$, then we have 
\begin{align*}
\|B^\top \Delta C\| = O(\epsilon), &\quad 
\|B^\top \Delta\| \le O(\frac{\epsilon}{\sigma_{min}(C)}), \\ \|\Delta C\| \le O(\frac{\epsilon}{\sigma_{min}(B)}),&\quad\|\Delta\|\le O(\frac{\epsilon}{\sigma_{min}(B)\sigma_{min}(C)}).
\end{align*}
\end{lemma}

We will first prove Theorem~\ref{thm:whiten} assuming Lemma~\ref{lem:mainperturbterm}. 
\begin{proof} [Proof of Theorem \ref{thm:whiten}]
By Lemma~\ref{lem:represent}, we know $\Sigma_{ab}$ can be written as 
$$
(A+\Delta^1_A)R_{AB}(B+\Delta^1_B)^\top + \Delta_{AB}.
$$
Similarly $\Sigma_{ca}$ can be written as
$$
(C+\Delta^3_C)R_{CA}(A+\Delta^3_A)^\top + \Delta_{CA}.
$$
Here the $\Delta$ terms and $R$ terms are bounded as in Lemma~\ref{lem:represent}.

Now let us write the matrix $\Sigma_{ab}[\Sigma_{bc}^{\top}]_m^+\Sigma_{ca}$ as
$$
\left((A+\Delta^1_A)R_{AB}(B+\Delta^1_B)^\top + \Delta_{AB}\right)([CB^\top]^+ + \Delta_{BC})\left((C+\Delta^3_C)R_{CA}(A+\Delta^3_A)^\top + \Delta_{CA}\right)$$

We can now view $\Sigma_{ab}[\Sigma_{bc}^{\top}]_m^+\Sigma_{ca}$ as the product of three terms, each term is the sum of two matrices. Therefore we can expand the product into 8 terms. In each of the three pairs, we will call the first matrix the main matrix, and the second matrix the perturbation.

In the remaining proof, we will do calculations to show the product of the main terms is close to $AA^\top$, and all the other 7 terms are small.

Before doing that, we first prove several Claims about PSD matrices

\begin{claim}\label{clm:tool}
If $\|\Delta\| \le \epsilon$, then $A\Delta A^\top \preceq \epsilon AA^\top$. If $\|\Gamma\| \le \epsilon\sigma_{min}(A)$, then $\frac{1}{2}(A\Gamma^\top + \Gamma A^\top) \preceq \epsilon AA^\top + \epsilon\sigma_{min}^2(A)\Id$ .
\end{claim}

\begin{proof}
Both inequalities can be proved by consider the quadratic form. We know for any $x$, $x^\top A\Delta A^\top x \le \|\Delta\| \|A^\top x\|^2 \le \epsilon x^\top AA^\top x$, so the first part is true.

For the second part, for any $x$ we can apply Cauchy-Schwartz inequality
$$
x^\top \frac{1}{2}(A\Gamma^\top + \Gamma A^\top) x = \inner{\sqrt{\epsilon}A^\top x, \epsilon^{-1/2}\Gamma^\top x} \le \epsilon\|A^\top x\|^2+\epsilon^{-1}\|\Gamma^\top x\|^2 = x^\top (\epsilon AA^\top + \epsilon\sigma_{min}^2(A)\Id)x.
$$

\end{proof}

Now, we will first prove the product of three main matrices is close to $AA^\top$:

\begin{claim}
We have
$\left((A+\Delta^1_A)R_{AB}(B+\Delta^1_B)^\top\right) (CB^\top)^+\left((C+\Delta^3_C)R_{CA}(A+\Delta^3_A)^\top \right) = AA^\top + E_{A}$, where $E_{A}$ is $O(\epsilon)$-spectrally bounded by $AA^\top$.
\end{claim}

\begin{proof}
We will first prove the middle part of the matrix $(B+\Delta^1_B)^\top(CB^\top)^+(C+\Delta^3_C)$ is $O(\epsilon)$ close to identity matrix $\Id$. Here we observe that both $B,C$ have full column rank so $(CB^\top)^+ = (B^\top)^+C^+$. Therefore we can rewrite the product as $(\Id + B^+\Delta^1_B)^\top (\Id + C^+\Delta^3_C)$. Since $\|\Delta^1_B\| \le O(\epsilon\sigma_{min}(B))$ by Lemma~\ref{lem:represent} (and similarly for $C$),  we know $\|B^+\Delta^1_B\| \le O(\epsilon)$. Therefore the middle part is $O(\epsilon)$ close to $\Id$.  Now since $\epsilon\ll 1$ we know
$\widehat{R}_{AB} = R_{AB}(B+\Delta^1_B)^\top (CB^\top)^+(C+\Delta^3_C)R_{CA}$ is $O(\epsilon)$-close to $\Id$.

Now we are left with $(A+\Delta^1_A)\widehat{R}_{AB}(A+\Delta^3_A)^\top$, for this matrix we know
$$
(A+\Delta^1_A)\widehat{R}_{AB}(A+\Delta^3_A)^\top - AA^\top = A(\widehat{R}_{AB} - \Id)A^\top + \Delta^1_A\widehat{R}_{AB}A^\top + A\widehat{R}_{AB}(\Delta^3_A)^\top + \Delta^1_A\widehat{R}_{AB}(\Delta^3_A)^\top.
$$
The first term $A(\widehat{R}_{AB} - \Id)A^\top \preceq O(\epsilon)AA^\top$ (Claim~\ref{clm:tool}); the fourth term $\Delta^1_A\widehat{R}_{AB}(\Delta^3_A)^\top\preceq O(\epsilon\sigma_{min}^2(A))\Id$ (by the norm bounds of $\Delta^1_A$ and $\Delta^3_A$. For the cross terms, we can bound them using the second part of Claim~\ref{clm:tool}.
\end{proof}

Next we will try to prove the remaining 7 terms are small. We partition them into three types depending on how many $\Delta$ factors they have. We proceed to bound them in each of these cases.

For the terms with only one $\Delta$, we claim:

\begin{claim}
The three terms
$\Delta_{AB}(CB^\top)^+(C+\Delta^3_C)R_{CA}(A+\Delta^3_A)^\top$, $(A+\Delta^1_A)R_{AB}(B+\Delta^1_B)^\top \Delta_{AB} (C+\Delta^3_C)R_{CA}(A+\Delta^3_A)^\top $, $\left((A+\Delta^1_A)R_{AB}(B+\Delta^1_B)^\top\right)(CB^\top)^+\Delta_{CA}$ are all $O(\epsilon)$ spectrally bounded by $AA^\top$.
\end{claim}

\begin{proof}
For the first term, note that both $B,C$ have full column rank, and hence $(CB^\top)^+ = (B^\top)^+C^+$. Therefore the first term can be rewritten as 
$$
[\Delta_{AB}(B^\top)^+] [(\Id + C^+\Delta^3_C)R_{CA}] (A+\Delta^3_A)^\top.
$$
\sloppy By Lemma~\ref{lem:represent}, we have spectral norm bounds for $\Delta_{AB}, \Delta^3_C, \Delta^3_A, R_{CA}$. Therefore we know $\|\Delta_{AB}(B^\top)^+\|\le O(\epsilon\sigma_{min}(A))$ and $[(\Id + C^+\Delta^3_C)R_{CA}]$ is $O(\epsilon)$ close to $\Id$. Therefore $\|[\Delta_{AB}(B^\top)^+] [(\Id + C^+\Delta^3_C)R_{CA}] (\Delta^3_A)^\top\| \le O(\epsilon\sigma_{min}^2(A))$ is trivially $O(\epsilon)$ spectrally bounded, and $[\Delta_{AB}(B^\top)^+] [(\Id + C^+\Delta^3_C)R_{CA}]A^\top$ is $O(\epsilon)$ spectrally bounded by Claim~\ref{clm:tool}. The third term is exactly symmetric.

For the second part, we will first prove the middle part of the matrix $\widehat{\Delta}_{BC} = (B+\Delta^1_B)^\top\Delta_{BC}(C+\Delta^3_C)$ has spectral norm $O(\epsilon)$. This can be done y expanding it to the sum of 4 terms, and use appropriate spectral norm bounds on $\Delta_{BC}$ and its products with $B^\top$ and $C$ from Lemma~\ref{lem:mainperturbterm}. Now we can show $(A+\Delta^1_A)R_{AB} \widehat{\Delta}_{BC}R_{CA}(A+\Delta^3_A)^\top$ is $O(\epsilon)$ spectrally bounded by the first part of Claim~\ref{clm:tool}.
\end{proof}

Next we try to bound the terms with two $\Delta$ factors. 

\begin{claim}
The three terms
$\Delta_{AB}\Delta_{BC}(C+\Delta^3_C)R_{CA}(A+\Delta^3_A)^\top$, $\Delta_{AB} (CB^\top)^+ \Delta_{CA}$,\\ $\left((A+\Delta^1_A)R_{AB}(B+\Delta^1_B)^\top\right)\Delta_{BC}\Delta_{CA}$ are all $O(\epsilon^2)$ spectrally bounded by $AA^\top$.
\end{claim}

\begin{proof}
For the first term, notice that $\|\Delta_{BC}(C+\Delta^3_C)\|$ is bounded by $O(\epsilon/\sigma_{min}(B))$ by Lemma~\ref{lem:mainperturbterm}, and $\|\Delta_{AB}\| = O(\epsilon\sigma_{min}(A)\sigma_{min}(B))$. Therefore we know $\|\Delta_{AB}\Delta_{BC}(C+\Delta^3_C)R_{CA}\| \le O(\epsilon^2 \sigma_{min}(A))$, so by Claim~\ref{clm:tool} we know this term is $O(\epsilon^2)$ spectrally bounded by $AA^\top$. Third term is symmetric.

For the second term, by Lemma~\ref{lem:represent} we can directly bound its spectral norm by $O(\epsilon^2\sigma_{min}^2(A))$, so it is trivially $O(\epsilon^2)$ spectrally bounded by $AA^\top$.
\end{proof}

Finally, for the product $\Delta_{AB}\Delta_{BC}\Delta_{CA}$, we can get the spectral norm for the three factors by Lemma~\ref{lem:represent} and Lemma~\ref{lem:mainperturbterm}. As a result $\|\Delta_{AB}\Delta_{BC}\Delta_{CA}\| \le O(\epsilon^3 \sigma^2_{min}(A))$ which is trivially $O(\epsilon^3)$ spectrally bounded by $AA^\top$.

Combining the bound for all of the terms we get the theorem.
\end{proof}

With that, we now try to prove Lemma~\ref{lem:mainperturbterm}
\newcommand{\tE}{\tilde{E}}

We first prove a simpler version where the perturbation is simply bounded in spectral norm

\begin{lemma}\label{lem:maintermsimple}
Suppose $B,C$ are $n\times m$ matrices and $n \ge m$. Let $R$ be an $n\times n$ matrix such that $\|R-\Id\|\le \epsilon$, and $E$ is a perturbation matrix with $\|E\| \le \epsilon\sigma_{min}(B)\sigma_{min}(C)$ and $(CRB^\top + E)$ is also of rank $m$.

Now let $\Delta = (CRB^\top+E)^+ - (CRB^\top)^+$, then when $\epsilon \ll 1$ we have 
\begin{align*}
\|B^\top \Delta C\| = O(\epsilon), &\quad 
\|B^\top \Delta\| \le O(\frac{\epsilon}{\sigma_{min}(C)}), \\ \|\Delta C\| \le O(\frac{\epsilon}{\sigma_{min}(B)}),&\quad\|\Delta\|\le O(\frac{\epsilon}{\sigma_{min}(B)\sigma_{min}(C)}).
\end{align*}

\begin{proof}
We first give the proof for $\|B^\top \Delta C\|$. Other terms are similar.

Let $U_B$ be the column span of $B$, and $U'_B$ be the row span of $(CRB^\top+E)$. Similarly let $U_C$ be the column span of $C$ and $U'_C$ be the column span of $(CRB^\top+E)$. By Wedin's theorem, we know $U'_B$ is $O(\epsilon)$ close to $U_B$ and $U'_C$ is $O(\epsilon)$ close to $U_C$. As a result, suppose the SVD of $B$ is $U_BD_BV_B^\top$, we know 
$$\sigma_{min}(B^\top U'_B) = \sigma_{min}(V_BD_B U_B^\top U'_B) \ge (1-O(\epsilon)) \sigma_{min}(B).$$
The same is true for $C$: $\sigma_{min}(C^\top U'_C) \ge (1-O(\epsilon)) \sigma_{min}(C)$.

By the property of pseudoinverse, the column span of $(CRB^\top+E)^+$ is $U'_B$, and the row span of $(CRB^\top+E)^+$ is $U'_C$, further, $(CRB^\top+E)^+ = U'_B[(U'_C)^\top (CRB^\top+E)U'_B]^{-1} U'_C$, therefore we can write

$$
B^\top (CRB^\top+E)^+ C = B^\top U'_B[(U'_C)^\top (CRB^\top+E)U'_B]^{-1} (U'_C)^\top C.
$$
Note that now the three matrices are all $n\times n$ and invertible! We can write $B^\top U'_B = ((B^\top U'_B)^{-1})^{-1}$ (and do the same thing for $(U'_C)^\top C$. Using the fact that $P^{-1}Q^{-1} = (QP)^{-1}$, we have
\begin{align*}
B^\top (CRB^\top+E)^+ C & = (((U'_C)^\top C)^{-1} (U'_C)^\top (CRB^\top+E)U'_B(B^\top U'_B)^{-1})^{-1} \\
& = (R + ((U'_C)^\top C)^{-1} (U'_C)^\top E U'_B(B^\top U'_B)^{-1})^{-1}  =: (R+X)^{-1}.
\end{align*}

Here we defined $X = ((U'_C)^\top C)^{-1} (U'_C)^\top E U'_B(B^\top U'_B)^{-1}$. The spectral norm of $X$ can be bounded by
\begin{align*}
\|X\| & \le \|((U'_C)^\top C)^{-1}\| \|E\|\|(B^\top U'_B)^{-1}\|
\\ & = \|E\|\sigma_{min}^{-1}(B^\top U'_B)\sigma_{min}^{-1}(C^\top C'_B) 
\\ & \le O(\epsilon).
\end{align*}
We can write $B^\top \Delta C = B^\top (CRB^\top+E)^+ C - \Id = (\Id + (R-\Id+X))^{-1} - \Id$, and we now know $\|(R-\Id+X)\| \le O(\epsilon)$, as a result $\|B^\top \Delta C \|\le O(\epsilon)$ as desired.

For the term $\|B^\top \Delta\|$, by the same argument we have
\begin{align*}
B^\top (CRB^\top+E)^+  & = ((U'_C)^\top (CRB^\top+E)U'_B(B^\top U'_B)^{-1})^{-1} (U'_C)^\top \\
& =  ((U'_C)^\top C R + (U'_C)^\top E U'_B(B^\top U'_B)^{-1})^{-1}(U'_C)^\top \\
& =  ((U'_C)^\top C (R+X))^{-1} (U'_C)^\top\\
& = (R+X)^{-1} ((U'_C)^\top C)^{-1}(U'_C)^\top.
\end{align*}

On the other hand, we know $B^\top (CRB^\top)^+ = R^{-1}C^+ = R^{-1} ((U_C)^\top C)^{-1} U_C^\top$. We can match the three factors:
\begin{align*}
\|R^{-1} - (R+X)^{-1}\| \le O(\epsilon), &\quad \|R^{-1}\| \le 1+O(\epsilon) \\
\|((U_C)^\top C)^{-1} - ((U'_C)^\top C)^{-1}\| \le O(\epsilon/\sigma_{min}(C)), & \quad \|((U_C)^\top C)^{-1}\| = O(1/\sigma_{min}(C)) \\
\|U_C-U'_C\| \le O(\epsilon), & \quad \|U_C\| = 1.
\end{align*}
Here, first and third bound are proven before. The second bound comes if we consider the SVD of $C = U_CD_CV_C^\top$ and notice that $\|(U'_C)^\top U_C - \Id\| \le O(\epsilon)$. We can write $\Delta_1 = R^{-1} - (R+X)^{-1}$, $\Delta_2 = ((U_C)^\top C)^{-1} - ((U'_C)^\top C)^{-1}$, $\Delta_3 = U_C-U'_C$, then we have
\begin{align*}
B^\top \Delta & = B^\top (CRB^\top+E)^+ - B^\top (CRB^\top)^+ \\ & =( R^{-1}-\Delta_1) (((U_C)^\top C)^{-1} - \Delta_2) (U_C-\Delta_3)^\top -  R^{-1} ((U_C)^\top C)^{-1} U_C^\top.
\end{align*}

Expanding the last equation, we get 7 terms and all of them can be bounded by $O(\epsilon/\sigma_{min}(C))$. The bounds on $\|\Delta C\|$ and $\|\Delta\|$ can be proved using similar techniques.
\end{proof}

\end{lemma}

Finally we are ready to prove the main Lemma~\ref{lem:mainperturbterm}:

\begin{proof}[Proof of Lemma~\ref{lem:mainperturbterm}]
Using Lemma~\ref{lem:represent}, let $E = E_{bc}^\top$, we can write the matrix before pseudoinverse as
$$
[CB^\top+E]_m = (C+\Delta_C)R_{BC}(B+\Delta_B)^\top + \Delta_{BC}.
$$

We can then apply Lemma~\ref{lem:maintermsimple} on $(C+\Delta_C)R_{BC}(B+\Delta_B)^\top + \Delta_{BC}$. As a result, we know if we let $\Delta' = [CB^\top+E]_m^+ - ((C+\Delta_C)R_{BC}(B+\Delta_B)^\top)^+$, we have the desired bound if we left multiply with $(B+\Delta_B)^\top$ or right multiply with $(C+\Delta_C)$.

We will now show how to prove the first bound, all the other bounds can be proved using the same strategy:

First, we can write 
$$B^\top \Delta' C = -(B+\Delta_B)^\top \Delta' (C+\Delta_C) + (B+\Delta_B)^\top \Delta' C + \Delta_B^\top \Delta' (C+\Delta_C) - \Delta_B^\top \Delta' \Delta_C.$$
All the four terms on the RHS can be bounded by Lemma~\ref{lem:maintermsimple} so we know $\|B^\top \Delta' C\| \le O(\epsilon)$.

On the other hand, let $\Delta'' = ((C+\Delta_C)R_{BC}(B+\Delta_B)^\top)^+ - (C^\top B)^+ = \Delta - \Delta'$. We will prove $\|B^\top \Delta''C\| \le O(\epsilon)$ and then the bound on $\|B^\top \Delta C\|$ follows from triangle inequality.

For $B^\top \Delta''C$, we know it is equal to
$$
B^\top [(B+\Delta_B)^\top]^+ R_{AB}^{-1} (C+\Delta_C)^+C - \Id
$$

\begin{claim}
$\|B^\top [(B+\Delta_B)^\top]^+ R_{AB}^{-1} (C+\Delta_C)^+C - \Id\| \le O(\epsilon)$
\end{claim}
\begin{proof}
We will show all three factors in the first term are $O(\epsilon)$ close to $\Id$. For $R_{AB}^{-1}$ this follows  immediately from Lemma~\ref{lem:represent}. For $(C+\Delta_C)^+C$, we know 
$$
(C+\Delta_C)^+C - \Id = - (C+\Delta_C)^+\Delta_C.
$$
Therefore its spectral norm bound is bounded by $\|\Delta_C\|\sigma_{min}^{-1}(C+\Delta_C) = O(\epsilon)$ (where the bound on $\|\Delta_C\|$ comes from Lemma~\ref{lem:represent}).
\end{proof}

With the claim we have now proven $\|B^\top \Delta''C\| \le O(\epsilon)$, therefore
$$
\|B^\top \Delta C\| \le \|B^\top (\Delta'+\Delta'')C\| \le \|B^\top \Delta'C\|+\|B^\top\Delta'' C\| \le O(\epsilon).
$$

\end{proof}

\subsection{Spectrally Boundedness and Incoherence}
\label{sec:incoherence}

Here we will show under mild incoherence conditions (defined below), if an error matrix $E$ is $\epsilon$-spectrally bounded by $FF^\top$, then the partial matrices satisfy the requirement of Theorem~\ref{thm:whiten}.

\begin{theorem}
If $F$ is $\mu$-incoherent for $\mu \le \sqrt{n/m \log^2 n}$, then when $n \ge \Omega(m\log^2 m)$, with high probability over the random partition of $F$ into $A,B,C$, we know $\sigma_{min}(A) \ge \sigma_{min}(F)/3$ (same is true for $B, C$).

As a corollary, if $E$ is $\epsilon$-spectrally bounded by $F$. Let $a,b,c$ be the subsets corresponding to $A,B,C$, and let $E_{a,b}$ be the submatrix of $E$ whose rows are in set $a$ and columns are in set $b$. Then $E_{a,b}$ (also $E_{b,c},E_{c,a}$) is $O(\epsilon)$-spectrally bounded by the corresponding asymmetric matrices $AB^\top$ ($BC^\top$, $CA^\top$).
\label{t:asymmetricspec}
\end{theorem}

\begin{proof} Consider the singular value decomposition of $F$: $F = UDV^\top$. Here $U$ is a $n\times m$ matrix whose columns are orthonormal, $V$ is an $m\times m$ orthonormal matrix and $D$ is a diagonal matrix whose smallest diagonal entry is $\sigma_{min}(F)$.

Consider the following way of partitioning the matrix: for each row of $F$, we put it into $A,B$ or $C$ with probability $1/3$ independently.

Now, let $X_i = 1$ if row $i$ is in the matrix $A$, and 0 otherwise. Then $X_i$'s are Bernoulli random variables with probability $1/2$. Suppose $S$ is the set of rows in $A$, let $U_A$ be $U$ restricted to rows in $A$, then we have $A = U_A DV^\top$. We will show with high probability $\sigma_{min}(A) \ge 1/3$.

The key observation here is the expectation of $U_A^\top U_A = \sum_{i=1}^n X_i U_i U_i^\top$, where $U_i$ is the $i$-th row of $U$ (represented as a column vector). Since $X_i$'s are Bernoulli random variables, we know

$$
\E[U_A^\top U_A] = \E[\sum_{i=1}^n X_i U_i U_i^\top] = \frac{1}{3} \sum_{i=1}^n U_i U_i^\top = \frac{1}{3}\Id.
$$

Therefore we can hope to use matrix concentration to prove that $U_A^\top U_A$ is close to its expectation.

Let $M_i = X_i U_iU_i^\top - 1/3 U_iU_i^\top$. Clearly $\E[M_i] = 0$. By the Incoherence assumption, we know $\|U_i\| \le 1/\log n$. Therefore we know $\|M_i\| \le O(1/\log n)$. Also, we can bound the variance
$$
\|\E[\sum_{i=1}^n M_iM_i^\top]\| \le \|\E[\sum_{i=1}^n X_i U_iU_i^\top U_iU_i^\top]\| \le \max \|U_i\|^2\|\sum_{i=1}^n U_iU_i^\top\| \le O(1/\log^2 n).
$$
Here the last inequality is because $\sum_{i=1}^n X_i U_iU_i^\top U_iU_i^\top \preceq \|U_i\|^2 U_iU_i^\top$.

Therefore by Matrix Bernstein's inequality we know with high probability $\|\sum_{i=1}^n M_i\| \le 1/6$. When this happens we know
$$
\|U_A^\top U_A\| \ge \sigma_{min}(\E[U_A^\top U_A]) - \|\sum_{i=1}^n M_i\| \ge 1/6.
$$
Hence we have $\sigma_{min}(U_A) \ge \sqrt{1/6} > 1/3$, and $\sigma_{min}(A) \ge \sigma_{min}(U_A)\sigma_{min}(D) \ge \sigma_{min}(F)/3$. Note that matrices $B$, $C$ have exactly the same distribution as $A$ so the bounds for $B$,$C$ follows from union bound.

For the corollary, if a matrix $E$ is $\epsilon$ spectrally bounded, we can write it as $F\Delta_1 F^\top + F\Delta_2^\top + \Delta_2 F^\top + \Delta_4$, where $\|\Delta_1\|\le \epsilon$, $\|\Delta_2\| \le \epsilon \sigma_{min}(F)$ and $\|\Delta_4\|\le \epsilon \sigma_{min}^2(F)$. This can be done by considering different projections of $E$: let $U$ be the span of columns of $F$, then $F\Delta_1 F^\top$ term corresponds to $\Proj_U E \Proj_U$; $F\Delta_2^\top$ term corresponds to $\Proj_U E \Proj_{U^\perp}$; $\Delta_2 F^\top$ term corresponds to $\Proj_{U^\perp} E \Proj_{U}$; $\Delta_4$ term corresponds to $\Proj_{U^\perp} E \Proj_{U^\perp}$. The spectral bounds are necessary for $E$ to be spectrally bounded.

Now for $E_{a,b}$, we can write it as $A\Delta_1 B^\top + 
A (\Delta_2)_b^\top + (\Delta_2)_a B^\top + (\Delta_4)_{a,b}$, where we also take the corresponding submatrices of $\Delta$'s. Since the spectral norm of a submatrix can only be smaller, we know $\|\Delta_1\| \le \epsilon$, $\|(\Delta_2)_b\| \le \epsilon\sigma_{min}(F) \le 3\epsilon \sigma_{min}(B)$, $\|(\Delta_2)_a\| \le \epsilon\sigma_{min}(F) \le 3\epsilon\sigma_{min}(A)$ and $\|(\Delta_2)_{a,b}\| \le \epsilon\sigma_{min}^2(F) \le 9\epsilon\sigma_{min}(A)\sigma_{min}(B)$. Therefore by Definition~\ref{def:asymmetric} we know $E_{a,b}$ is $9\epsilon$ spectrally bounded by $AB^\top$.
\end{proof}

\section{Proof of Theorem~\ref{t:finalguarantee} and Theorem~\ref{t:finaldeterministic}} 
\label{s:missing_final} 

In this section, we provide the full proof of Theorem~\ref{t:finalguarantee}. 
We start with a simple technical Lemma. 
\begin{lemma}
	If $Q$ is an $\epsilon$-approximate whitening matrix for $A$, then 
	$\|Q\| \leq \frac{1}{1-\epsilon} \|AA^{\top}\|$, $\sigma_{\min}(Q) \geq \frac{1}{1-\epsilon} \|AA^{\top}\|$
	\label{l:approx_evals}  
\end{lemma} 
\begin{proof}
	
	By the definition of approximate-whitening, we have 
	$$1 - \epsilon \leq \sigma_{\min}( (Q^+)^{1/2} A^{\top} A (Q^+)^{1/2}), \sigma_{\max}( (Q^+)^{1/2} A^{\top} A (Q^+)^{1/2}) \leq 1+\epsilon$$ which implies that 
	$$1 - \epsilon \leq \sigma_{\min}( (Q^+)^{1/2} A A^{\top} (Q^+)^{1/2}), \sigma_{\max}( (Q^+)^{1/2} A A^{\top} (Q^+)^{1/2}) \leq 1+\epsilon$$ 
	\sloppy by virtue of the fact that $(Q^+)^{1/2} A A^{\top} (Q^+)^{1/2} = \left((Q^+)^{1/2} A^{\top} A (Q^+)^{1/2}\right)^{\top}$. Rewriting in semidefinite-order notation, 
	we get that
	$$(1-\epsilon) \Id \preceq (Q^+)^{1/2} A A^{\top} (Q^+)^{1/2} \preceq (1+\epsilon) \Id $$ 
	Multiplying on the left and right by $Q^{1/2}$, we get  
	$$(1-\epsilon) Q \preceq AA^{\top} \preceq (1+\epsilon) Q $$  
	This directly implies $\frac{1}{1+\epsilon} AA^{\top} \preceq Q \preceq \frac{1}{1-\epsilon} AA^{\top}$
	which is equivalent to the statement of the lemma. 
\end{proof} 
Towards proving Theorem~\ref{thm:main-random}, we will first prove the following proposition, which shows that we recover the $\exp(-W)$ matrix correctly:   

\begin{proposition}[Recovery of $\exp(-W)$] Under the random generative model defined in Section~\ref{s:intro}, if the number of samples satisfies

$$N =\poly(n,1/p/,1/\rho) $$ 
the vectors  $\tilde{W}_{i}, i \in [m]$ in Algorithm~\ref{alg:main} are 
 $O(\eta \sqrt{n p} )$-close to $\exp(-W_i)$  where 
$$\eta  = \tilde{O}\left( \sqrt{m \spr}  \rho \right) $$   

\label{p:beforelog}
\end{proposition}
\begin{proof} 

The proof will consist of checking the conditions for Algorithms \ref{a:alg1} and \ref{a:alg2} to work, so that we can apply Theorems \ref{thm:main_alg} and \ref{thm:whiten}.   

To get a handle on the PMI tensor, by Proposition~\ref{p:firstspectral}, for any equipartition $S_a, S_b, S_c$ of $[n]$, we can it as
\begin{align*} & \pmitensor_{S_a, S_b, S_c} = \\ 
& \frac{\prior}{1-\prior} \sum_{k \in [m]} F_{k,S_a} \otimes F_{k,S_b} \otimes F_{k,S_c} + 
\sum_{l = 2}^L (-1)^{l+1} \left(\frac{1}{l}\left(\frac{\prior}{1-\prior}\right)^l\right) \sum_{k \in [m]} (\tilde{\taylor}_l)_{k,S_a} \otimes (\tilde{\taylor}_l)_{k,S_b} \otimes (\tilde{\taylor}_l)_{k,S_c} + E_L \end{align*}

We can choose $\displaystyle L = \poly(\log(n,\frac{1}{\rho},\frac{1}{p}))$ to ensure 
\begin{equation} \|E_L\| = o\left(p^{5/2} \rho^{7/3} \sqrt{m} n^2\right) \label{eq:mb1} \end{equation}

Having an explicit form for the tensor, we proceed to check the spectral boundedness condition for Algorithm \ref{a:alg1}.

Let $S_a, S_b, S_c$ be a random equipartition. Let $R_{S_a}$ be the matrix that has as columns the vectors $\displaystyle \left(\frac{1}{l}\left(\frac{\prior}{1-\prior}\right)^l\right)^{1/3} (\tilde{\taylor}_l)_{j, S_a}$, for all $l \in [2,L], j \in [m]$ and let $A$ be the matrix that has as columns the vectors $\left({\frac{\prior}{1-\prior}}\right)^{1/3} F_{j, S_a}$, for all $j \in [m]$. Since $L$ is polynomially bounded in $n$, by Proposition~\ref{p:thirdspectral} we have that with high probability $R_{S_a}$ is $\tau$-spectrally bounded by $A$, for a $\displaystyle \tau = O(\prior^{2/3} \log n)$. Analogous statements hold for $S_b, S_c$. 

Next, we verify the conditions for calculating approximate whitening matrices (Algorithm \ref{a:alg1}). 

Towards applying Theorem~\ref{t:asymmetricspec}, note that if $R_{[n]}$ is the matrix that has as columns the vectors $\displaystyle \left(\frac{1}{l}\left(\frac{\prior}{1-\prior}\right)^l\right)^{1/3} (\tilde{\taylor}_l)_{j}$, for all $l \in [2,L], j \in [m]$, and $D$ is the matrix that has as columns the vectors $\left(\frac{\prior}{1-\prior}\right)^{1/3} F_{j}$, for all $j \in [m]$, $R_{[n]}$ then is $\tau$-spectrally bounded by $D$ for $\displaystyle \tau = O(\prior^{2/3} \log n)$. Furthermore, the matrix $F$ is $O(1)$-incoherent with high probability by Lemma \ref{l:rowbound}. 
Hence, we can apply Theorem~\ref{t:asymmetricspec}, the output of Algorithm \ref{a:alg1} are matrices $Q_a, Q_b, Q_c$ which are $\tau$-approximate whitening matrices for $A,B,C$ respectively.

Next, we will need bounds on 
$$ \min(\sigma_{\min}(Q_a), \sigma_{\min}(Q_b), \sigma_{\min}(Q_c)),  \max(\sigma_{\max}(Q_a), \sigma_{\max}(Q_b), \sigma_{\max}(Q_c))$$ 
to plug in the guarantee of Algorithm \ref{a:alg1}. 

By Lemma~\ref{l:approx_evals}, we have 
$$\sigma_{\max}(Q_a) \leq \frac{1}{1-\tau} \|AA^{\top}\| \lesssim (1+\tau) \|AA^{\top}\|, \sigma_{\min}(Q_a) \geq \frac{1}{1+\tau} \sigma_{\min}(AA^{\top}) \gtrsim (1-\tau) \sigma_{\min}(AA^{\top})$$ 
However, for the random model, applying Lemma~\ref{l:boundevals}, 
\begin{equation} \sigma_{\min}(AA^{\top}) \geq \left(\frac{\prior}{1-\prior}\right)^{2/3} n \spr \gtrsim \prior^{2/3} n \spr,  \sigma_{\max}(AA^{\top}) \leq \left(\frac{\rho}{1-\rho}\right)^{2/3} m n \spr^2 \lesssim \rho^{2/3} m n \spr^2 \label{eq:mb2}\end{equation}  
Analogous statements hold for $B$ and $C$.

Finally, we bound the error due to empirical estimates. Since $\prior \spr m = o(1)$, 
$$\Pr[s_i = 0 \wedge s_j = 0 \wedge s_k = 0] \geq 1 - \Pr[s_i = 1] - \Pr[s_j = 1] - \Pr[s_k = 1] \geq 1 - 3 \spr m \prior = \Omega(1)$$  
Hence, by Corollary \ref{c:samplespmi}, with a number of samples as stated in the theorem, 
\begin{equation} \|\hat{\pmitensor}_{S_a,S_b,S_c} - \pmitensor_{S_a,S_b,S_c}\|_{\{1,2\},\{3\}} \lesssim p^{5/2} \rho^{5/2} \sqrt{m} n^2 \label{eq:mb3}\end{equation}
as well. 

With that, invoking Theorem \ref{thm:main_alg} (with $\|E\|_{\{1,2\},\{3\}}$ taking into account both the $E_L$ term above, and the above error due to sampling), the output of Algorithm~\ref{a:alg2} will produce vectors $v_i$, $i \in [m]$, s.t. $v_i$ is $O(\eta')$-close to $\left(\frac{\rho}{1-\rho}\right)^{1/3} (1 - \exp(-W_i))$, for 

$$\eta' \lesssim  \max(\Norm{Q_a}, \Norm{Q_b}, \Norm{Q_c})^{1/2}  \cdot \left(\tau^{3/2}  + \sigma^{-3/2}(\|E_L\|_{\{1,2\},\{3\}} + \|\hat{\pmitensor}_{S_a,S_b,S_c} - \pmitensor_{S_a,S_b,S_c}\|_{\{1,2\},\{3\}})\right) $$

where $\sigma = \min(\sigma_{\min}(Q_a), \sigma_{\min}(Q_b), \sigma_{\min}(Q_c))$. 

Plugging in the estimates from \eqref{eq:mb1}, \eqref{eq:mb2}, \eqref{eq:mb3} as well as $\tau = O(\rho^{2/3} \log n)$, we get:
\begin{align*} 
\max(\Norm{Q_a}, \Norm{Q_b}, \Norm{Q_c})^{1/2} \tau^{3/2} & \lesssim \sqrt{m n \spr^2 \prior^{2/3}} (\prior^{2/3} \log n)^{3/2} = \prior^{4/3} \sqrt{mn} p \log^{3/2}n \\ 
\sigma^{-3/2}\|E_L\|_{\{1,2\},\{3\}} & \lesssim (\frac{1}{\rho n \spr})^{3/2} \|E_L\|_{\{1,2\},\{3\}} \lesssim \prior^{4/3} \sqrt{mn} p \log^{3/2}n \\
\sigma^{-3/2} \|\hat{\pmitensor}_{S_a,S_b,S_c} - \pmitensor_{S_a,S_b,S_c}\|_{\{1,2\},\{3\}} & \lesssim  \prior^{4/3} \sqrt{mn} p \log^{3/2}n 
\end{align*}
 
which implies the vectors $(\hat{a}_i, \hat{b}_i, \hat{c}_i)$ are $O(\eta')$-close to $\left(\frac{\rho}{1-\rho}\right)^{1/3} (1-\exp(W_i))$, for all $i \in [m]$. 

However, this directly implies that 
$\left(\frac{1-\rho}{\rho}\right)^{1/3} (\hat{a}_i, \hat{b}_i, \hat{c}_i)$ are $O( \eta'/\rho^{1/3}) $-close to $1-\exp(W_i)$, $i \in [m]$, which in turn implies $(\tilde{a}_i,\tilde{b}_i,\tilde{c}_i)$ are $O( \eta'/\rho^{1/3}) $ close to $\exp(W_i)$. 

This implies the statement of the Lemma. 

\end{proof}

Given that, we prove the main theorem. The main issue will be to ensure that taking $\log$ of the values 

\begin{proof}[Proof of Theorem~\ref{t:finalguarantee}]

By Proposition~\ref{p:beforelog}, the vectors $Y_i, i \in [m]$ in Algorithm~\ref{alg:main} are 
 $O(\eta \sqrt{n p})$-close to $\exp(-W_i)$ with $\eta = \tilde{O}\left( \sqrt{m \spr} \rho\right)$ . 
 Let $(Y_i')_j = (Y_i)_j$ if $(Y_i)_j > \exp(-\nu_u) $ and otherwise $(Y_i)_j' = \exp(-\nu_u)$. 
 
 Then we have that $\norm{Y_i'-W_i}\le \norm{Y_i-W_i}$. 
 By the Lipschitzness of $\log(\cdot)$ in the region $[\nu_i,\infty]$
 
 we have that 
 \begin{align}
|(\widehat{W}_i)_j - (W_i)_j| = |\log (Y_i')_j-\log (W_i)_j|\lesssim  |(Y_i)_j'-(W_i)_j|\nonumber
 \end{align}
 
It follows that 
\begin{align}
\norm{\widehat{W}_i - W_i} = \norm{\log Y_i'-\log W_i}\lesssim  \norm{Y_i'-W_i}\nonumber
\end{align}

Therefore recalling $\norm{Y_i'-W_i}\le \norm{Y_i-W_i}\le O(\eta \sqrt{n p})$ we complete the proof. 
\\
\end{proof}

\begin{proof}[Proof of Theorem~\ref{t:finaldeterministic}]

The proof will follow the same outline as the proof of Theorem~\ref{t:finalguarantee}. The difference is that since we only have a guarantee on the spectral boundedness of the second and third-order term, we will need to bound the higher-order terms in a different manner. Given that we have no information on them in this scenario, we will simply bound them in the obvious manner. We proceed to formalize this. 

The sample complexity is polynomial for the same reasons as in the proof of Theorem~\ref{t:finalguarantee}, so we will not worry about it here.

We only need to check the conditions for Algorithms \ref{a:alg1} and \ref{a:alg2} to work, so that we can apply Theorems \ref{thm:main_alg} and \ref{thm:whiten}.

Towards that, first we claim that we can write the PMI tensor for any equipartition $S_a, S_b, S_c$ of $[n]$ as
\begin{align} & \pmitensor_{S_a, S_b, S_c} = \nonumber\\ 
& \frac{\prior}{1-\prior} \sum_{k \in [m]} F_{k,S_a} \otimes F_{k,S_b} \otimes F_{k,S_c} 
- \left(\frac{1}{2}\left(\frac{\prior}{1-\prior}\right)^2\right) \sum_{k \in [m]} G_{k,S_a} \otimes G_{k,S_b} \otimes G_{k,S_c}  \nonumber\\
& + \left(\frac{1}{3}\left(\frac{\prior}{1-\prior}\right)^3\right) \sum_{k \in [m]} H_{k,S_a} \otimes H_{k,S_b} \otimes H_{k,S_c} +E \label{eq:tensor} \end{align}

where $\|E\|_{\{1,2\},\{3\}} \leq \rho^4 m (np)^{3/2}$. Towards achieving this, first we claim that Proposition~\ref{lem:tensor_norm_bound} implies that for any subsets $S_a, S_b, S_c$, 
\begin{align*} \left\|\sum_{k=1}^m \left(1-\exp(-l \bw_k)\right)_{S_a} \otimes \left(1-\exp(-l \bw_k)\right)_{S_b} \otimes \left(1-\exp(-l \bw_k)\right)_{S_c}\right\|_{\{1,2\},\{3\}} \leq m(np)^{3/2}
\end{align*} 
Indeed, if we put $\gamma_k =  \left(1-\exp(-l \bw_k)\right)_{S_a}$, $\delta_k = \left(1-\exp(-l \bw_k)\right)_{S_b}$, $\theta_k = \left(1-\exp(-l \bw_k)\right)_{S_c}$, then we have 
$\|\sum_k \gamma_k \gamma_k^{\top} \| \leq \sqrt{ m n p} $, and similarly for $\delta_k$. Since $\max_k \|\theta_k\| \leq (np)^{1/2}$, the claim immediately follows. 
Hence, \eqref{eq:tensor} follows.

Next, let $R_{S_a}$ be the matrix that has as columns the vectors $\left(\frac{1}{l}\left(\frac{\prior}{1-\prior}\right)^l\right)^{1/3} (\tilde{\taylor}_l)_{j,S_a}$, $l \in [2,L], j \in [m]$ and $A$ is the matrix that has as columns the vectors $\left(\frac{\prior}{1-\prior}\right)^{1/3} (\tilde{\taylor}_1)_{j,S_a}$ for $j \in [m]$ for some $L = O(\poly(n))$, similarly as in the proof of Theorem \ref{t:finalguarantee}.  

We claim that $R_{S_a}R^{\top}_{S_a}$ is $\tau$ spectrally bounded by $\rho FF^{\top}$. 

Indeed, for any $l > 2$, we have $\|\left(\frac{1}{l}\left(\frac{\prior}{1-\prior}\right)^l\right)^{1/3} \tilde{\taylor}_l\| \lesssim \prior^{l/3} \|\tilde{\taylor}_l\| \lesssim \prior^{l/3} \sqrt{m n} p$
Hence, 
\begin{align}
R_{S_a}R^{\top}_{S_a} & \preceq \rho^{2/3} GG^{\top}  + \rho^{4/3} HH^{\top} + \rho^2 LL^{\top} + \sum_{l \geq 4} \prior^{2l/3} m n p^2  \nonumber\\
& \preceq 3 \rho^{2/3} \tau (F F^{\top} + \sigma_{\min}(FF^{\top})) + \prior^{8/3} m n p^2 \precsim \rho^{2/3} \tau (F F^{\top} + \sigma_{\min} (FF^{\top}))
\end{align}
where the first inequality holds since $HH^{\top}, GG^{\top}, LL^{\top}$ are $\tau$-\spb bounded by $F$ and the second since $\sigma_{\min}(FF^{\top})\gtrsim n p$ and $\tau \ge 1$. 
Let $\tau' = \rho^{2/3} \tau$. 
Since we are assuming the matrix $F$ is $O(1)$-incoherent, we can apply Theorem~\ref{t:asymmetricspec}, and claim the output of Algorithm \ref{a:alg1} are matrices $Q_a, Q_b, Q_c$ which are $\tau$-approximate whitening matrices for $A,B,C$ respectively.

By Lemma~\ref{l:approx_evals}, we have again
$$\sigma_{\max}(Q_a) \leq \frac{1}{1-{\tau'}} \lesssim (1+{\tau'}) \|AA^{\top}\|, \sigma_{\min}(Q_a) \geq \frac{1}{1+{\tau'}} \gtrsim (1-{\tau'}) \sigma_{\min}(AA^{\top})$$ 

Then, applying Theorem~\ref{thm:main_alg}, we get that we recover vectors 
$(\hat{a}_i, \hat{b}_i, \hat{c}_i)$ are $O(\eta')$-close to $\left(\frac{\rho}{1-\rho}\right)^{1/3} (1-\exp(W_i))$, for all $i \in [m]$. 
for
$$\eta' \lesssim  \max(\Norm{Q_a}, \Norm{Q_b}, \Norm{Q_c})^{1/2}  \cdot \left({\tau'}^{3/2}  + \sigma^{-3/2}\|E\|_{\{1,2\},\{3\}}\right) $$
Recall that $\tau' = \rho^{2/3} \tau$, and $\norm{Q_a}\le \rho^{2/3} \sigma_{\max}(F) \lesssim \rho^{2/3} \sqrt{mn} p$ and $\|E\|_{\{1,2\},\{3\}} \le \rho^4 m(np)^{3/2}$ and $\sigma \gtrsim \rho^{2/3} np$, we obtain that 
$$\eta'\lesssim \rho^{1/3} \sqrt{mn} p \left((\tau \rho^{2/3})^{3/2} + \frac{\rho^4 m (np)^{3/2}}{(\rho^{2/3}np)^{3/2}} \right) \lesssim \rho^{4/3} \sqrt{m n} p \tau^{3/2}  $$ 
where the last inequality holds since $\rho^{3} m = o(1) = o(\tau)$.  

However, this directly implies that 
$(\frac{1-\rho}{\rho})^{1/3} (\hat{a}_i, \hat{b}_i, \hat{c}_i)$ are $O( \eta'/\rho^{1/3}) = O(\eta)$-close to $1-\exp(-W_i)$, $i \in [m]$, which in turn implies $(\tilde{a}_i,\tilde{b}_i,\tilde{c}_i)$ are $O(\eta)$ close to $\exp(-W_i)$. 

Argument for recovering $W_i$ from $\exp(-W_i)$ is then exactly the same as the one in Theorem \ref{t:finalguarantee}.

\end{proof}

\section{Sample complexity and bias of the PMI estimator}\label{sec:sample}

Finally, we consider the issue of sample complexity. 
The estimator we will use for the PMI matrix will simply be the plug-in estimator, namely:
\begin{equation}\hat{\pmi}_{i,j} = \log \frac{\hat{\Pr}[s_i=0 \wedge s_j = 0]}{\hat{\Pr}[s_i = 0] \hat{\Pr}[s_j = 0]}\label{eq:pluginest}\end{equation}
 Notice that this estimator is biased, but as the number of samples grows, the bias tends to zero. Formally, we can show:

\begin{lemma}  If the number of samples $\samples$ satisfies 
$$\samples \geq \frac{1}{\min_{i \neq j}\{\Pr[s_i=0 \wedge s_j = 0\}} \frac{1}{\delta^2} \log m$$ 
with high probability  $|\hat{\pmi}_{i,j} - \pmi_{i,j}| \leq \delta, \forall i \neq j$.   
\label{l:samplematrix}
\end{lemma} 
\begin{proof}
Denoting 
$\Delta_{i,j} = \hat{\Pr}[s_i=0 \wedge s_j = 0] - \Pr[s_i=0 \wedge s_j = 0]$ and $\Delta_i =  \hat{\Pr}[s_i=0] - \Pr[s_i=0]$, we get that   
\begin{align*}
\hat{\pmi}_{i,j} & = \log \frac{\hat{\Pr}[s_i=0 \wedge s_j = 0]}{\hat{\Pr}[s_i = 0] \hat{\Pr}[s_j = 0]} = \log \frac{\Pr[s_i=0 \wedge s_j = 0] + \Delta_{i,j}}{(\Pr[s_i = 0] + \Delta_{i}) (\Pr[s_j = 0]+\Delta_j)} \\ 
&= \pmi_{i,j} + \log\left(1 + \frac{\Delta_{i,j}}{\Pr[s_i=0 \wedge s_j = 0]}\right) - \log\left(1+\frac{\Delta_i}{\Pr[s_i = 0]}\right) - \log\left(1+\frac{\Delta_j}{\Pr[s_j=0]}\right) \\ 
\end{align*} 

Furthermore, we have that $\frac{2x}{2+x} \leq \log(1+x) \leq \frac{x}{\sqrt{x+1}}$, for $x \geq 0$, which implies that when $x \leq 1$, 
$\frac{2}{3} x \leq \log(1+x) \leq x$. From this it follows that if 
$$ \max\left(\max_{i,j} \frac{\Delta_{i,j}}{\Pr[s_i=0 \wedge s_j = 0]}, \max_i \frac{\Delta_i}{\Pr[s_i = 0]}\right) \leq \delta$$ 
we have 
$$\pmi_{i,j} - \frac{\delta}{3} \leq \hat{\pmi}_{i,j} \leq \pmi_{i,j} + \delta $$  
 
Note that it suffices to show that if $\samples >  \frac{1}{1 - 4\spr_{\max} m \prior_{\max}} \frac{1}{\delta^2} \log m$, we have 
\begin{equation}
\Pr\left[\frac{\Delta_i}{\Pr[s_i = 0]} > (1+\delta) \vee \frac{\Delta_i}{\Pr[s_i = 0]} < (1-\delta)\right] \leq \exp(-\log^2m) 
\label{eq:sampchernoff_singleton} 
\end{equation}
and 
\begin{equation}
\Pr\left[\frac{\Delta_{i,j}}{\Pr[s_i = 0 \wedge s_j=0]} > (1+\delta) \vee \frac{\Delta_{i,j}}{\Pr[s_i = 0 \wedge s_j=0]}  < (1-\delta)\right] \leq \exp(-\log^2m) 
\label{eq:sampchernoff_pairs} 
\end{equation}
since this implies 
$$ \max\left(\max_{i,j} \frac{\Delta_{i,j}}{\Pr[s_i=0 \wedge s_j = 0]}, \max_i \frac{\Delta_i}{\Pr[s_i = 0]}\right) \leq \delta$$ 
with high probability by a simple union bound. 

Both \eqref{eq:sampchernoff_singleton} and \eqref{eq:sampchernoff_pairs} will follow by a Chernoff bound. 

Indeed, consider \eqref{eq:sampchernoff_singleton} first. We have by Chernoff
$$\Pr\left[\Delta_{i} > \left(1+\sqrt{\frac{\log \samples}{\samples \Pr[s_i=0] } }\right) \Pr[s_i=0]\right] \leq \exp(-\log^2 \samples)$$ 
Hence, if $\samples > \frac{1}{\Pr[s_i=0] } \frac{1}{\delta^2} \log m$, we get that $1-\delta \leq \frac{\Delta_i}{\Pr[s_i = 0]} \leq 1+\delta$ with probability at least $1 - \exp(\log^2 m)$. 

The proof of \eqref{eq:sampchernoff_pairs} is analogous -- the only difference being that the requirement is that $\samples > \frac{1}{\Pr[s_i=0] } \frac{1}{\delta^2} \log m$
which gives the statement of the lemma.

\end{proof} 

Virtually the same proof as above shows that: 

\begin{lemma}  If the number of samples $\samples$ satisfies 
$$\samples \geq \frac{1}{\min_{i \neq j \neq k}\{\Pr[s_i=0 \wedge s_j = 0 \wedge s_k=0\}} \frac{1}{\delta^2} \log m$$ 
with high probability
$|\hat{\pmitensor}_{i,j,k} - \pmitensor_{i,j,k}| \leq \delta, \forall i \neq j \neq k$. 
\end{lemma} 

As an immediate corollary, we get: 
\begin{corollary} If the number of samples $\samples$ satisfies 
$$\samples \geq \frac{1}{\min_{i \neq j \neq k}\{\Pr[s_i=0 \wedge s_j = 0 \wedge s_k=0\}} \frac{1}{\delta^2} \log m$$ 
\label{c:samplespmi}
$$\samples \geq \frac{1}{\min_{i \neq j \neq k}\{\Pr[s_i=0 \wedge s_j = 0 \wedge s_k=0\}} \frac{1}{\delta^2} \log m$$ 
with high probability for any equipartition $S_a, S_b, S_c$ 
$$\|\hat{\pmitensor}_{S_a, S_b, S_c} - \pmitensor_{S_a, S_b, S_c}\|_{\{1,2\},\{3\}} \lesssim n^3 \delta $$ 
\end{corollary}

\section{Matrix Perturbation Toolbox}

In this section we discuss standard matrix perturbation inequalities.
Many results in this section can be found in Stewart and Sun~\cite{stewart1977perturbation}.
Given $\wh{A} = A + E$, the perturbation in individual singular values can be bounded by Weyl's theorem:

\begin{theorem}[Weyl's theorem]\label{thm:weyl}
Given $\wh{A} = A+E$, we know $\sigma_k(A) - \|E\| \le \sigma_k(\wh{A}) \le \sigma_k(A) + \|E\|$.
\end{theorem}

For singular vectors, the perturbation is bounded by Wedin's Theorem:

\begin{lemma}[Wedin's theorem; Theorem 4.1, p.260 in \cite{stewart1990matrix}]
  \label{lem:perturb-subspace}
  Given matrices $A, E\in\R^{m\times n }$ with $m\ge n$. Let $A$ have the singular value
  decomposition
  \begin{align*}
    A = [U_1, U_2, U_3]\left[
    \begin{array}[c]{cc}
      \Sigma_1 & 0 \\ 0 & \Sigma_2 \\ 0 & 0
    \end{array}
  \right]
  [V_1, V_2]^\top.
  \end{align*}
  Let $\wh A = A + E$, with analogous singular value decomposition. Let $\Phi$ be the matrix of
  canonical angles between the column span of $U_1$ and that of $\wh U_1$, and $\Theta$ be the matrix of
  canonical angles between the column span of $V_1$ and that of $\wh V_1$. Suppose that there exists
  a $\delta$ such that
  \begin{align*}
    \min_{i,j}|[\Sigma_{1}]_{i,i}-[\Sigma_2]_{j,j}|>\delta,\quad \mbox{and } \quad \min_{i,i}|[\Sigma_1]_{i,i}|> \delta,
  \end{align*}
  then
  \begin{align*}
    \|\sin(\Phi)\|^2 + \|\sin(\Theta)\|^2 \le 2{\|E\|^2\over \delta^2}.
  \end{align*}
\end{lemma}

\paragraph{Perturbation bound for pseudo-inverse}

When we have a lowerbound on $\sigma_{min}(A)$, it is easy to get bounds for the perturbation of pseudoinverse.

\begin{theorem}[Theorem 3.4 in \cite{stewart1977perturbation}]
  \label{sec:pert-bound-pseudo}
Consider the perturbation  of a matrix $A\in\R^{m\times n}$: $B = A+E$. Assume that $rank(A)
=rank(B)= n$, then
\begin{align*}
  \|B^\dag - A^\dag\|\le \sqrt{2}\|A^\dag\|\|B^\dag\| \|E\|.
\end{align*}
\end{theorem}

Note that this theorem is not strong enough when the perturbation is only known to be $\tau$-spectrally bounded in our definition.

\end{document}